\documentclass[journal]{IEEEtran}
\usepackage{bm}
\usepackage{bbm}
\usepackage{amsmath}
	\usepackage{leftidx}
\usepackage{amsthm}
\usepackage{amssymb}
\usepackage{mathtools}
\usepackage{enumerate}
\usepackage{multicol}
\usepackage{subfig}
\usepackage{float}
\usepackage{graphicx}
\usepackage{tikz}
\usetikzlibrary{shapes.geometric, arrows, positioning}
\usetikzlibrary{cd}
\usepackage[all]{xy}
\usepackage{color}
\usepackage{verbatim}
\usepackage[makeroom]{cancel} %use: has cancel-strikethrough macro
\usepackage{faktor}
\usepackage{subfiles} %Used to combine multiple files
\usepackage{soul}
\usepackage{graphicx}
\usepackage{imakeidx}
\makeindex[title=Notes,columns=1]

\usepackage{pstricks-add}
\usepackage{pdftricks}
\begin{psinputs}
  \usepackage{pstricks}
  \usepackage{multido}
\end{psinputs}
\usepackage{mathrsfs}
\DeclareMathAlphabet{\mathpzc}{OT1}{pzc}{m}{it}

\newcommand{\opn}[1]{\operatorname{#1}}						%Operators
	
\newtheorem{theorem}{Theorem}[section]
\newtheorem*{theorem*}{Theorem}				    			%Theorem
\newtheorem{lemma}[theorem]{Lemma}			                %Lemma
\newtheorem{?}[theorem]{Question}	                        %Question

\newtheorem{proposition}[theorem]{Proposition}				%Proposition
\newtheorem{corollary}[theorem]{Corollary}					%Corollary
\newtheorem*{assumption}{Assumption}                        %Assumption

\theoremstyle{remark} % Changes style of all following newtheorems e.g. def or remark
\newtheorem{remark}{Remark}
\newtheorem*{notation}{Notation}
				
\newcommand{\shortminus}[0]{\scalebox{0.75}[1.0]{\( - \)}}
\newcommand{\fr}[2]{\frac{#1}{#2}}							%Fractions
\newcommand{\paran}[1]{\left(#1\right)}	                    %Paranthesis
\newcommand{\set}[1]{\left\{#1\right\}}				        %Set
\newcommand{\abs}[1]{\left|#1\right|}	                    %Absolute Value
\newcommand{\norm}[1]{\left\|#1\right\|}					%Norm
\newcommand{\bv}[1]{\bm{#1}}								%Bold Vector
\newcommand{\hatbv}[1]{\hat{\bv {#1}}}				    	%Unit Vector #1
\newcommand{\uv}[1]{\hat{\mathbf{#1}}}						%Unit Vector #2
\newcommand{\guv}[1]{\hat{\bm{#1}}}							%Greek Unit Vector
\renewcommand{\vec}[1]{\langle #1 \rangle}

\newcommand{\C}{\MC{C}}

\newcommand{\RR}[0]{\mathbb{R}}

\newcommand{\hquad}{\hspace{0.5em}}

\let\tilde\widetilde 				
\newenvironment{definition}[1][Definition.]{\begin{trivlist}		
\item[\hskip \labelsep {\bfseries #1}]}{\end{trivlist}}		%Definition

\renewcommand{\C}[1]{\cos( \theta_{#1} )}
\renewcommand{\S}[1]{\sin( \theta_{#1} )}

\newrgbcolor{ffqqtt}{1 0 0.2}
\newrgbcolor{ttttff}{0.2 0.2 1}
\psset{xunit=1cm,yunit=1cm,algebraic=true,dimen=middle,dotstyle=o,dotsize=5pt 0,linewidth=1.6pt,arrowsize=3pt 2,arrowinset=0.25}

\usepackage{xcolor, hyperref}
\definecolor{orcidlogocol}{HTML}{A6CE39}

\DeclareRobustCommand{\orcidicon}{%
	\begin{tikzpicture}
	\draw[lime, fill=lime] (0,0) 
	circle [radius=0.16] 
	node[white] {{\fontfamily{qag}\selectfont \tiny ID}};
	\draw[white, fill=white] (-0.0625,0.095) 
	circle [radius=0.007];
	\end{tikzpicture}
	\hspace{-2mm}
}

\newcommand{\orcid}[1]{\href{https://orcid.org/#1}{\textcolor[HTML]{A6CE39}{\orcidicon}}}

\begin{document}

\title{A Geometric Approach to the Kinematics of the Canfield Joint}

\author{Christian~Bueno\orcid{0000-0002-3903-7615}, %~\IEEEmembership{Member,~IEEE,}
        Kristina~V.~Collins\orcid{0000-0002-3816-1948}, %~\IEEEmembership{Fellow,~OSA,}
        Alan Hylton, %~\IEEEmembership{Member,~IEEE,}
        and~Robert~Short\orcid{0000-0003-2087-5528}%,~\IEEEmembership{Life~Fellow,~IEEE}% <-this % stops a space
\thanks{Christian Bueno is in the Department of Mathematics, University of California, Santa Barbara, CA 93106. email: christianbueno@ucsb.edu}
\thanks{Kristina Collins is in the Department of Electrical, Computer, and Systems Engineering, Case Western Reserve University, Cleveland, Ohio, 44106 USA. email: kvc2@case.edu}% <-this % stops a space
\thanks{Alan Hylton and Robert Short are with NASA Glenn Research Center, Cleveland, Ohio, 44135 USA. email: alan.g.hylton@nasa.gov, robert.s.short@nasa.gov}}% <-this % stops a space

% The paper headers
% \markboth{Draft as of \today}%
% {Shell \MakeLowercase{\textit{et al.}}: Bare Demo of IEEEtran.cls for IEEE Journals}

% make the title area
\maketitle

\begin{abstract}
This paper details an accessible geometric derivation of the forward and inverse kinematics of a parallel robotic linkage known as the Canfield joint, which can be used for pointing applications. The original purpose of the Canfield joint was to serve as a human wrist replacement, and it can be utilized for other purposes such as the precision pointing and tracking of antennas, telescopes, and thrusters. We build upon previous analyses, and generalize them to include the situation where one of the three legs freezes; the kinematics are also substantially generalized beyond failure modes, detailed within. The core of this work states and clarifies the assumptions necessary to analyze this type of parallel robotic linkage.  Specific guidance is included for engineering use cases. 

\end{abstract}

% \begin{IEEEkeywords}
% kinematics, Canfield joint, parallel robot
% \end{IEEEkeywords}

\IEEEpeerreviewmaketitle

\section{Introduction}\label{sec:intro}

The Canfield joint is a parallel robotic linkage with three degrees of freedom (DoF)  \cite{patent}. Originally designed to mimic a wrist joint, it has been explored for deep space optical communications \cite{KCThesis, canfield_validation} and other end effectors, including solar cells and thrusters.

The Canfield joint bears some resemblance to a delta robot, but has a 3-\underline{R}UR kinematic chain rather than a delta robot's 3-\underline{R}UU one.\footnote{Here \underline{R}, R, and U denote an actuated revolute joint, passive revolute joint, and passive universal joint respectively. See \cite{parallelcontrolbook}.} The position and orientation of its distal plate, on which an end effector is mounted, is controlled entirely by three actuators, one at the base of each leg. Developed in 1997 \cite{patent}, it is intended for use as a gimbal mechanism.  Compared to a traditional three-axis concentric gimbal, the Canfield joint offers superior size, weight, and power (SWaP) characteristics: because of its parallel load paths, it can articulate a larger end effector mass than a gimbal of comparable size, and cables can be routed through its center, obviating the need for an additional cable tray.

The Canfield joint also offers a hemispherical workspace and robust pointing operation when reduced to two degrees of freedom. One of the key results of this paper is to elaborate on this robustness. In the event that one of the legs locks into place, such as due to motor failure, we show it is often still possible to precisely point at celestial objects. Along the way to this result, we solve other inverse kinematics problems and formalize, generalize, and simplify existing approaches.  The robustness of the joint under this problem has been explored in prior work \cite{BobTopCJPaper, ChristianPoster}, but we give a more precise and effective description for inverse kinematics here. 

Despite being a parallel linkage, the Canfield joint has elegant kinematics when midplane symmetry is assumed. Given this assumption, we show it is possible to rigorously and completely solve various kinematic problems in a more general setting than previously addressed. In particular, the assumption of midplane symmetry enables us to study a generalized model of the Canfield joint which allows for variation of relative dimensions, including non-equilateral plates and unequal leg lengths. With the assumption of midplane symmetry, we can readily identify singularities of the forward kinematic mapping. Lastly, we demonstrate new relationships between quantities of interest and provide novel methods and visualizations. 

\section{Model Construction and Conventions}\label{sec:construction}

To obtain rigorous results, we must first introduce a model of the Canfield joint as a spatial kinematic chain. We will do this in two steps: First we will introduce what we term a ``half-joint'' and second we define the Canfield joint as two half-joints glued together along their corresponding free ends.
\begin{figure}[ht]
\centering
\includegraphics[width=0.45\textwidth]{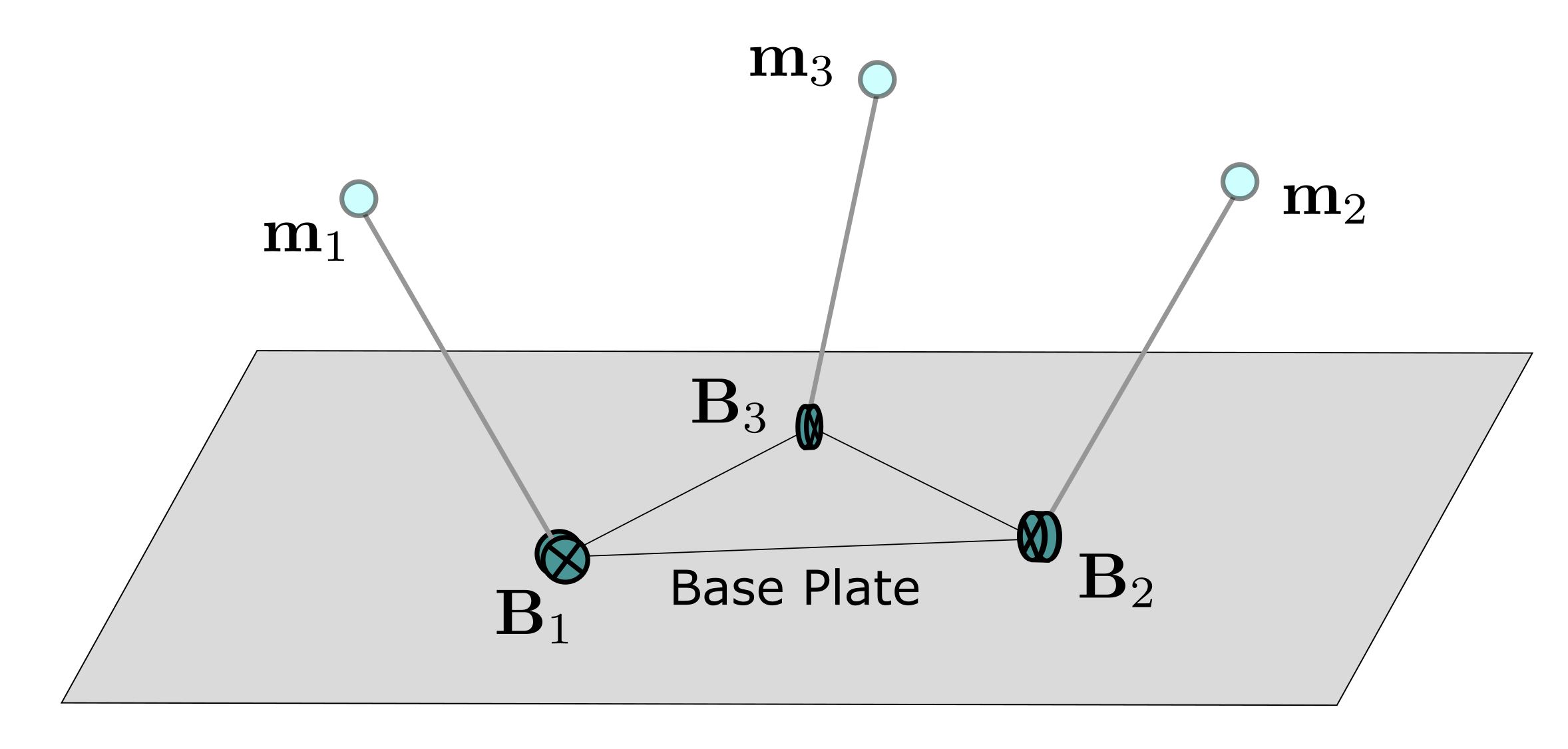}
\caption{Picture of a half-joint.} 
\label{fig:halfjoint}
\end{figure}

\begin{definition}
A \textbf{half-joint} is a triangular plate of nonzero area with links of nonzero length attached at each corner via a revolute joint (Figure \ref{fig:halfjoint}). The free endpoints are termed \textbf{midjoints} and are denoted $\bv m_1$, $\bv m_2$, and $\bv m_3$. 
\end{definition}

Note that in the above definition, there is no restriction on how the revolute joint is oriented. Indeed its axis of rotation is free to be chosen, but fixed thereafter. This will lead to a substantial generalization of the mechanism introduced in \cite{Canfield1998}.

\begin{figure}[ht]
	\centering
	\includegraphics[width=.45\textwidth]{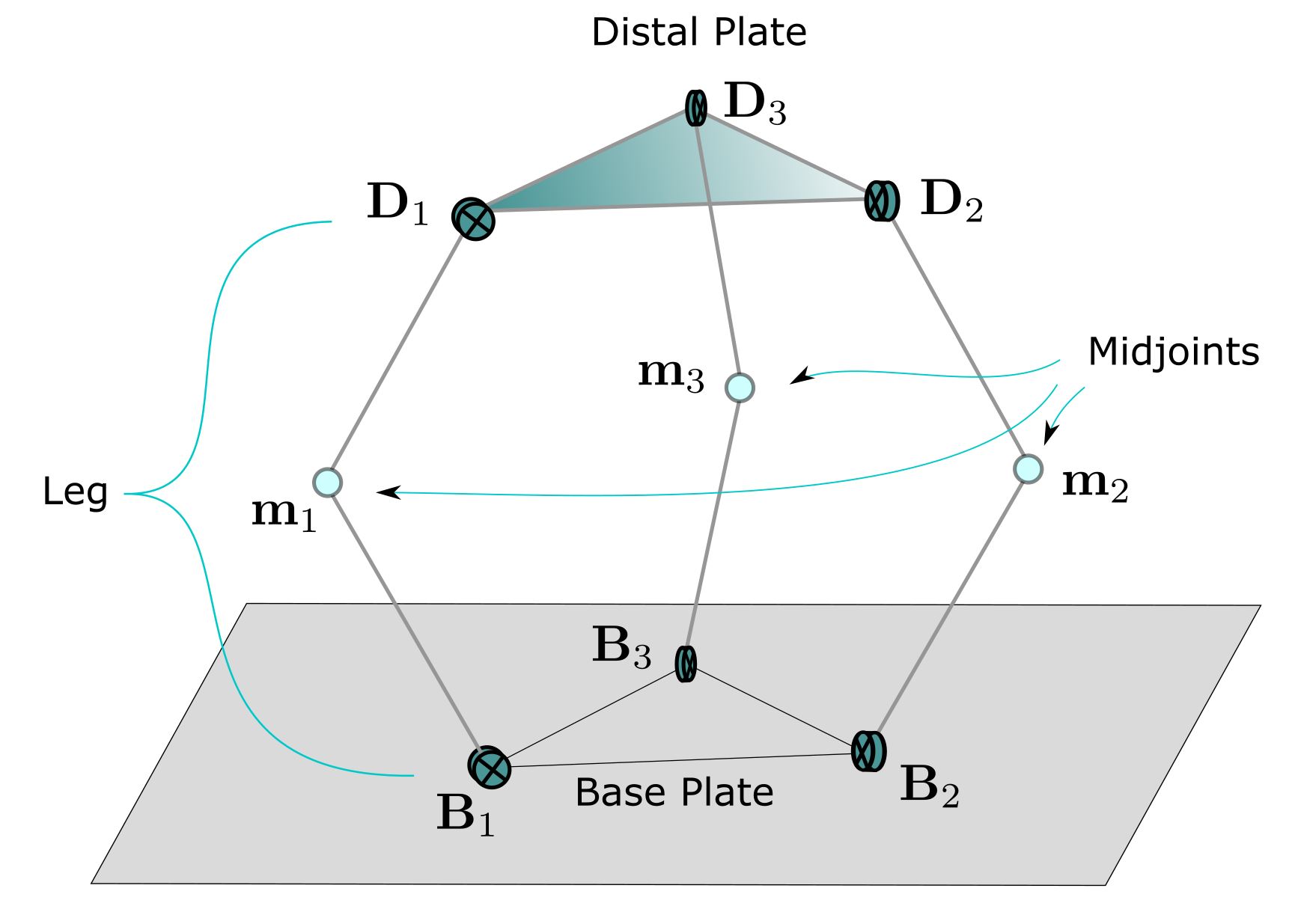}
	\captionsetup{justification=centering}
	\caption{Illustration of Canfield joint and variables.}
	\label{fig:canfieldjoint}
\end{figure}
\begin{figure}[ht]
\centering
\includegraphics[width=0.35\textwidth]{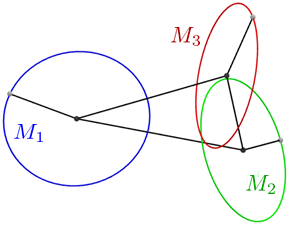}
\caption{Midcircles, relative to half-joint.} 
\label{fig:midcircles}
\end{figure}

\begin{definition}
A \textbf{Canfield joint} consists of two identical half-joints attached at their corresponding midjoints via spherical\footnote{The original prototype presented in \cite{Canfield1998} uses a universal joint which induces a slightly more restricted range of motion.} joints (Figure \ref{fig:canfieldjoint}). We designate one of the half-joints as the \textbf{base half-joint} and refer to the other one as the \textbf{distal half-joint}.
\begin{itemize}
    \item The \textbf{base plate} (\textbf{distal plate})\footnote{Many implementations place the hinges on the lateral midpoints of the triangular plates instead of the corners. Nevertheless, these hinges will form a triangle and it's these hinge triangles which we call the base/distal plates.} is the triangular plate on the base (distal) half-joint and the plane containing it is called the \textbf{base plane} (\textbf{distal plane}). 
    \item The \textbf{base hinges} and \textbf{distal hinges} are the revolute joints on the base plate and distal plate. Their locations are denoted by $\bv B_i$ and $\bv D_i$ for $i=1,2,3$ respectively ordered counterclockwise and such that $\bv B_i$ and $\bv D_i$ corresponds to $\bv m_i$. Only the base hinges are driven and we denote their axes of revolution by $\uv N_i$ for $i=1,2,3$ respectively.
    \item The $i$-th \textbf{midcircle} $M_i$ is the circle of all possible positions for the midjoint $\bv m_i$ (Figure \ref{fig:midcircles}). The plane containing $M_i$ is called the $i$-th \textbf{midcircle plane} and is denoted $P_i$. 
    % \item Each midjoint is constrained to move on a circle called a \textbf{midcircle} denoted by $M_1$, $M_2$, and $M_3$ respectively (see Figure \ref{fig:midcircles}). The plane containing $M_i$ is called the $i$-th midcircle plane and is denoted $P_i$. 
    \item The \textbf{base normal} $\uv N_B$ and \textbf{distal normal} $\uv N_D$ are unit vectors normal to the base and distal planes respectively, such that the hinges are numbered in counter-clockwise order relative to said vector's orientation. Each can be concretely given by,
    \[
        \uv N_B\!=\!\tfrac{(\!\bv B_2\shortminus\bv B_1\!)\times (\!\bv B_3\shortminus\bv B_1\!)}{\|(\!\bv B_2\shortminus\bv B_1\!) \times (\!\bv B_3\shortminus\bv B_1\!)\|},\quad
        \uv N_D\!=\!\tfrac{(\!\bv D_2\shortminus\bv D_1\!)\times (\!\bv D_3\shortminus\bv D_1\!)}{\|(\!\bv D_2\shortminus\bv D_1\!) \times (\!\bv D_3\shortminus\bv D_1\!)\|}.
    \]
    \item The \textbf{base backward-normal} and \textbf{distal backward-normal} are respectively given by $-\uv N_B$ and $-\uv N_D$. 
\end{itemize}
\end{definition}

In this paper, the \textbf{base center} $\bv B_c$ is a distinguished point which can be located anywhere on the base plane. The \mbox{{\textbf{distal center}}} $\bv D_c$ is the point on the distal plane whose position relative to the distal hinges is the same as the base center's position relative to the base hinges (i.e. the point corresponding to $\bv B_c$ in the distal plane).\footnote{In practice, the distal center is in the interior of the distal plate and is where the tool is located. However, we do not need to enforce such a restriction.} We will give a simpler characterization of the distal center in Proposition \ref{prop:forwardkinematics}.

For this paper, any right-handed XYZ coordinate system with origin $\bv B_c=\bv 0$, $z$-direction $\uv N_D =\uv z$, and base plane as the XY plane, will suffice. For a more concrete coordinate system, we offer the following notion of a \textbf{base frame}: Let $\uv x:=\tfrac{\bv B_2\shortminus\bv B_1}{\|\bv B_2\shortminus\bv B_1\|}$, $\uv z:=\uv N_B$, and $\uv y:=\uv z\times\uv x$. A \textbf{distal frame} can be defined analogously using the distal hinges and distal normal, but translated so that $\bv D_c=\bv 0$.

To simplify our analysis, we will allow different pieces of the mechanisms to intersect or even pass through one another.
We also assume all the elements are infinitely thin. Since there are many ways to construct a Canfield joint, we leave it up to the end-user to discard solutions which are physically impossible or irrelevant for their use-case.

% Old version
%In this model, it is possible for the different components to intersect %or even pass through one another. For example, it is possible for the %distal plate and base plate to intersect when the half-joints are %aligned. In this paper, we are allowing self-intersections in order to %permit as much freedom as possible.

%As a simplification for our analysis, we will allow for the links and plates to intersect and pass through one another. For example, starting with the configuration in Figure \ref{fig:canfieldjoint} it is permissable to flatten the whole Canfield joint by pushing the distal plate down onto the base plate. We can even push past this flattened configuration by pushing the distal plate below the base plate. By not worrying about these issues, we can focus our attention on understanding the inverse kinematics in the most generous settings. If a solution does not exist while allowing for self-intersections, then it won't exist for a physical model either. Conversely, if a solution does exists while allowing for self-intersection, we can readily check on paper or in a simulation whether this configuration is physically realizable. 

A Canfield joint model following the above definition requires three types of parameters to be specified: the distance from each hinge to its corresponding midjoint (denoted $\ell_i$ for $i =1,2,3$), the distances between adjacent base hinges (denoted $b_i$ for $i = 1,2,3$) and the orientation of the base hinge axes (introduced earlier as $\uv N_i$). This gives a total of $3+3+2\cdot3=12$ design parameters to be chosen.  Using the above notations, we can describe the $i$-th midcircle $M_i$ as the set of $\bv x$ satisfying the following constraints:
\begin{align*}
    \uv N_i\cdot(\bv x - \bv B_i) &= 0, \\
    \norm{\bv x - \bv B_i}^2 &= \ell_i^2.
\end{align*}

Note that this level of generality is not taken in the construction established in \cite{Canfield1998}. In the \textbf{standard Canfield joint}, the base and distal plates are equilateral triangles, each hinge axis is parallel to the triangle side opposite said hinge, all the distances from a hinge to its connected midjoint are the same (denoted $\ell$), and that the distances between any two adjacent hinges are the same (denoted $b$).  In this simplified setup, it is typical to choose $\bv B_c$ to be the center of the base equilateral plate. The base hinge positions can then be given by:
%As an example, we can situate the base hinges in this frame of reference with the following formulas and where the hinges rotate about the axis in the XY plane perpendicular to the hinge location vector. \CBnote{I'm confused by last sentence}
%\CBnote{Alternate form
\[
\bv B_1 = \tfrac{b}{\sqrt{3}}\uv e_1, \quad \bv B_2= \opn{Rot}_z(\pi/3) \bv B_1, \quad \bv B_3 = \opn{Rot}_z(\pi/3) \bv B_2.
\]
Where $\uv e_1$ is the standard basis vector in the X direction and $\opn{Rot}_z(\theta)$ rotates vectors by angle $\theta$ with $\uv z$ as the axis of rotation. To wit:  
\[
    \bv B_1\!= \!\vec{\tfrac{b}{\sqrt{3}},0,0},\quad
    \bv B_2\!=\! \vec{\shortminus\tfrac{b}{2\sqrt{3}},\tfrac{b}{2}, 0},\quad
    \bv B_3\!=\!\vec{\shortminus\tfrac{b}{2\sqrt{3}}, \shortminus\tfrac{b}{2}, 0}.
\]
% \begin{align*}
%     \bv B_1 &= \vec{\tfrac{b}{\sqrt{3}},0,0},\\
%     \bv B_2 &= \vec{-\tfrac{b}{2\sqrt{3}},\tfrac{b}{2}, 0}, \\
%     \bv B_3 &= \vec{-\tfrac{b}{2\sqrt{3}}, -\tfrac{b}{2}, 0}.
% \end{align*}

We will refer to this example on occasion throughout the paper.  However, the results we have do not depend on this setup to proceed.

% \begin{definition}\RSnote{Tri}
% A Canfield joint with parameters $\ell$ and $b$ has:
% \begin{itemize}
%     \item three base hinges located at
%     \begin{itemize}
%         \item $\bv B_1 = \vec{\tfrac{b}{\sqrt{3}},0,0}$,
%         \item $\bv B_2 = \vec{-\tfrac{b}{2\sqrt{3}},\tfrac{b}{2}, 0}$, and
%         \item $\bv B_3 = \vec{-\tfrac{b}{2\sqrt{3}}, -\tfrac{b}{2}, 0}$;
%     \end{itemize}
%     where each base hinge rotates about the axis in the $xy$-plane which is perpendicular to the hinge location vector;
%     \item the center of the base plate located at $\bv B_c = \vec{0,0,0}$;
%     \item three midjoints, with each connected by a leg of length $\ell$ to a base hinge;
%     \item three distal hinges, with each connected by a leg of length $\ell$ to a midjoint, and each distance $b$ from each other.
% \end{itemize}

% \end{definition}

The core of our treatment is the following notion of midplane symmetry.  We will use this idea throughout the paper.

\begin{definition}
A configuration of the Canfield joint has \textbf{midplane symmetry} if there is some plane passing through all three midjoints such that the reflection of the base half-joint over the the plane gives the distal half-joint.  Such a plane is called the \textbf{midplane} for that configuration.
\end{definition}

There are infinitely many planes that pass through all three midjoints if they are colinear. This can happen for certain design parameters, and can happen even in the case of the standard Canfield joint. However, if the configuration of the Canfield joint satisfies midplane symmetry, there is only one midplane for that configuration.  Moreover, when the midjoints are not colinear, they may be used directly to uniquely identify the midplane.  In this case, a normal vector to the midplane can be given by $\bv N = (\bv m_2 - \bv m_1) \times (\bv m_3 - \bv m_1)$.

Assuming midplane symmetry, specifying the base half-joint configuration and a plane through the midjoints determines a unique configuration of the Canfield joint with midplane symmetry. This unique configuration can be constructed by joining the base half-joint to its reflection across the provided plane. Conversely, given a configuration with midplane symmetry, the base half-joint configuration and midplane are uniquely determined (any two planes which cause the same reflection of the base half-joint must be the same). This correspondence allows us to focus on the midplane as an inverse kinematic tool which we informally describe below:
\[
    \begin{matrix}
    \text{Midplane-symmetric} \\
    \text{Canfield joint} \\
    \text{configurations}
    \end{matrix}
    \quad\longleftrightarrow\quad
    \begin{matrix}
    \text{Half-joint configurations} \\
    \text{equipped with a} \\
    \text{plane through midjoints}
    \end{matrix}
\]

Although not every possible configuration of a Canfield joint will satisfy midplane symmetry, these are the only configurations we have seen reliably utilized in practice \cite{Canfield1998,ganino,PowerCube,CJGoogleSite,BobTopCJPaper,HATTScomparison}.  Nevertheless, even with midplane symmetry, the cases where the midjoints are colinear allow for infinitely many choices of midplanes with which to complete the configuration. Such configurations are often associated with mechanical singularities that threaten a user's ability to control a Canfield joint. These are worth further study, as some configurations like these can be physically realized, however that is beyond the scope of this paper. Fortunately, these singularities seem to be rare and indeed, in \cite{Canfield1998}, Canfield claimed that the eponymous mechanism had ``A large, singularity free workspace.'' As such, we will assume the following:

\begin{assumption}
For the rest of this paper, we will assume that midplane symmetry is always satisfied by the configurations we are considering, unless otherwise stated.
\end{assumption}

% We will also see in Section \ref{sec:fwdkin} that we can always construct a configuration with midplane symmetry for a specified lower half-joint setting.

% \CBnote{Since midplane symmetry is paramount, we should probably make a clear axiom/assumption block stating that we only consider midplane symmetric situations. Or something that makes it very clear that this part is important to read and internalize. That or add the assumption to all places that use it. Also, worth noting that there is at most one midplane of symmetry, and possibly important to show reader how to construct midplane of symmetry if it exists.} \RSnote{Changed this paragraph.  Please reread and offer advice.  I think this critique has been (mostly) addressed.  Also, we give them a point and normal vector, so I don't think we need to show them how to construct a midplane.} \CBnote{That construction of the midplane fails sometimes, for example if the midjoints are colinear, so phrasing is pretty tricky}

\section{Forward Kinematics}\label{sec:fwdkin} Since only the base angles are driven, we need a way of determining the entire configuration from this information alone, i.e. the forward kinematics. In \cite{Canfield1998}, Canfield provides a description of the forward kinematics which implicitly uses midplane symmetry in an essential way. We will make this assumption explicit as well as reformulate and generalize the forward kinematics.

Our first task is to unite the base angles and midjoints.  The key to this are the midcircles $M_1$, $M_2$, and $M_3$. 
As an example, in the standard Canfield joint constructed in \cite{Canfield1998}, we can parametrize the midcircles using the base angles $\theta_1$, $\theta_2$, and $\theta_3$.  Using this parametrization, we can locate the midjoints using the formulas below:
% \begin{proposition}[Formulas for Midjoint Locations] \label{prop:midcircle_parametrization}\RSnote{Tri}
% Given a Canfield joint with parameters $\ell$ and $b$, the base angles $\theta_1$, $\theta_2$, and $\theta_3$, parametrize the corresponding midcircles.  Given specified angles, the following formulas locate the midjoints:
\begin{align*}
\bv m_1 &=  \vec{ \ell \C{1} + \tfrac{b}{\sqrt{3}}, 0 , \ell \S{1} },\\
\bv m_2 &= \vec{ -\tfrac{\ell}{2}\C{2} - \tfrac{b}{2\sqrt{3}}, \tfrac{\ell \sqrt{3}}{2}\C{2} + \tfrac{b}{2}, \ell \S{2} },\\
\bv m_3 &= \vec{ -\tfrac{\ell}{2} \C{3} - \tfrac{b}{2\sqrt{3}}, -\tfrac{\ell \sqrt{3}}{2}\C{3}-\tfrac{b}{2}, \ell \S{3} }.
\end{align*}

% \end{proposition}

% \begin{proof}
% Proof to go here
% \end{proof}

Since we can define the midjoints in terms of the base angles, it suffices to phrase our remaining results in terms of just the midjoints.  Detailed computations with specific values for $\ell_i, b_i$, and $\uv N_i$ are left to the interested reader.

% Combining this with the definition of the midplane above, we have the following corollary:

% \begin{corollary}[Formula for midplane]
% Given a Canfield joint with parameters $\ell$ and $b$ and base angles $\theta_1$, $\theta_2$, and $\theta_3$, an equation for the midplane is:
% \[
   
% \]
% % $$ Dis gun B uggly$$
% \end{corollary}

Since we are assuming midplane symmetry, the distal half-joint must be the reflection of the base half-joint over the midplane. This will allow us to readily locate features of the distal half-joint. To simplify exposition, we introduce the following notation:

\begin{notation}
We use $R_P(\bv a)$ to denote the reflection of a point $\bv a$ over a plane $P$. If the plane $P$ is given by $\bv N\cdot(\bv x- \bv q)=0$ then we may use $R_{[\bv N, \bv q]}(\bv a)$ to denote this reflection and it can be explicitly given by
\[ 
    R_P(\bv a) = R_{[\bv N, \bv q]}(\bv a)  = \bv a - 2\frac{\bv N \cdot (\bv a - \bv q)}{\norm{\bv N}^2}\bv N.
\]
\end{notation}
\begin{remark}\label{rk:reflection}
For convenience, we also note two special cases:
\begin{align*}
    R_{[\bv N, \bv 0]}(\bv a)  &=  \bv a - 2\frac{\bv N \cdot \bv a}{\norm{\bv N}^2}\bv N,\\%\,\, \mathrm{ and} \\
    R_{[\bv N,\bv q]}(\bv 0) &= 2\frac{\bv N \cdot \bv q}{\norm{\bv N}^2}\bv N,
\end{align*}
% \[
%     R_{[\bv N, \bv 0]}(\bv a)  =  \bv a - 2\frac{\bv N \cdot \bv a}{\norm{\bv N}^2}\bv N,
%     \quad 
%     R_{[\bv N,\bv q]}(\bv 0) = 2\frac{\bv N \cdot \bv q}{\norm{\bv N}^2}\bv N,
% \]
and a consequence, \[R_{[\bv N,\bv q]}(\bv a+\bv b)=R_{[\bv N,\bv q]}(\bv a)+R_{[\bv N,\bv 0]}(\bv b).\]
\end{remark}

Once the midplane has been defined, the distal half-joint is determined by reflecting the base plate over the midplane.  This gives us the following proposition:

\begin{proposition}[Forward Kinematics]\label{prop:forwardkinematics}
If a configuration has midplane $P$ which is normal to $\bv N$ and contains $\bv q$, then
\[
    \bv D_i = R_{[\bv N, \bv q]}(\bv B_i), \hquad \bv D_c = R_{[\bv N, \bv q]}(\bv 0), \hquad \uv N_D = R_{[\bv N,\bv 0]}(-\uv z).
\]
We can choose $\bv N=(\bv m_2-\bv m_1)\times(\bv m_3-\bv m_1)$ and $\bv q=\bv m_i$ for any $i=1,2,3$ if the midjoints are not colinear.
% If the midjoints are not colinear, then we can use $\bv q=\bv m_i$ for any $i=1,2,3$, and $\bv N=(\bv m_2-\bv m_1)\times(\bv m_3-\bv m_1)$.
% 
\end{proposition}

\begin{proof}
From definitions and midplane symmetry, we have that $R_P(\bv B_i) = \bv D_i$ and $R_P(\bv B_c)=R_P(\bv 0) = \bv D_c$.  This is what is meant by the phrase ``the point corresponding to $\bv B_c$ in the distal plane'' in the definition of $\bv D_c$.

Next, note $\uv N_D=(\uv N_D+\bv D_c)-\bv D_c$ is normal to the distal plane. Thus, by reflection symmetry, $R_P(\uv N_D+\bv D_c)-R_P(\bv D_c)$ is normal to the base plane. By the above, we know that $R_P(\bv D_c)=\bv B_c=\bv 0$, hence $R_P(\uv N_D+\bv D_c)$ is normal to the base plane. By Remark \ref{rk:reflection} we have that
\[
R_P(\bv D_c+\uv N_D)
= R_P(\bv D_c)+R_{[\bv N,\bv 0]}(\uv N_D)
= R_{[\bv N,\bv 0]}(\uv N_D).
\]
This must be $\pm\uv z$ since the above is normal to the base plane and is also a unit vector (since $\uv N_D$ was). Since reflections reverse orientations, the hinge numbering is now clockwise (instead of counter-clockwise) with respect to this vector. Thus, it is the case that $R_{[\bv N,\bv 0]}(\uv N_D)=-\uv N_B=-\uv z$. Applying $R_{[\bv N,\bv 0]}$ to both sides yields $\uv N_D=R_{[\bv N,\bv 0]}(-\uv z)$.

Finally, the midplane must contain all the midjoints by definition and if the midjoints are not colinear, the midplane is normal $\bv N=(\bv m_2-\bv m_1)\times(\bv m_3-\bv m_1)$.
\end{proof}

When the midjoints are not colinear, the midplane is uniquely determined and it is the one described in the above result. In this case, we can convert midjoints to a midplane and then get the distal features by reflection of base features. However, if the midjoints are colinear and the midplane is unknown, then the full configuration cannot be determined. This is a forward kinematic singularity which can be understood without appeal to traditional singularity analysis techniques involving Jacobians. %Indeed, the forward kinematic Jacobian would be undefined for colinear midjoints. 
Some examples of these singularities are discussed in \cite{KCThesis}. Isolating and identifying the importance of the midplane symmetry assumption substantially clarifies why midjoint colinearity results in a breakdown of the forward kinematics described in \cite{Canfield1998}.%Isolating and identifying the importance of the midplane symmetry assumption substantially clarifies the singularities which result when the midplane is unknown. %the area of the triangle described by the three midjoints is reduced to zero.

Because the distal plate location is entirely determined by midplane reflection, there is no rotation component that might ``twist'' the distal plate relative to the base plate. The only possible exception is the theoretical and unrealistic case in which the distal center and base center are in the same place. Thus, any rotation required for assets on the distal plate needs to be attached and run independently to the Canfield joint.

\subsection{Pointing the Distal Plate}

Once we have the distal plate constructed, we should determine what objects can be pointed at by the configuration. To do so, we define some notions that will prove useful to us in the inverse kinematics construction.

\begin{definition}
The \textbf{forward-pointing ray} is the ray starting at the distal center which points along the distal normal and the \textbf{backward-pointing ray} has the same initial point but opposite direction. We say a configuration of a Canfield joint ``\textbf{points at $\bv{T}$}'' if the forward-pointing ray contains $\bv T$ and ``\textbf{backward-points at $\bv T$}'' if the backward-pointing ray contains $\bv T$. The \textbf{pointing axis} is the union of these two rays.
\end{definition}

% \begin{lemma}\label{lem:distal_normal}\RSnote{Free}
% Given Canfield joint parameters $\ell$ and $b$ and base angles $\theta_1$, $\theta_2$, and $\theta_3$ symmetric over the midplane $M$, a normal vector to the distal plane $\uv N_D$ is given by:
% \[ \uv N_D = R_M(-\uv z ) - \bv D_c
% \]
% \end{lemma}

% \begin{proof}
% Notice that in our construction, $-\uv z - \bv B_c = -\uv z$ is a normal vector for the base plane.  Then, since the distal plane is the reflection of the base plane over the midplane $M$, it follows that a normal vector to the distal plane is given by $\uv N_D = R_M(-\uv z - \bv B_c)$.  Linearity and the fact that $\bv D_c = R_M(\bv B_c)$ then yields the result.
% \end{proof}

% , we say a midplane ``points at $\bv x$'' 
% if the forward-pointing ray from the distal plane contains $\bv T$. 

If we have a configuration of the Canfield joint with midplane $P$, the pointing axis can be given by
\[ 
\bv r (t) = \bv D_c + t \uv N_D, \quad t\in \RR
\]

The forward-pointing ray is the pointing axis with $t\geq 0$, and the backward-pointing ray is the pointing axis with $t\leq 0$.  

% \begin{proof}
% By the construction from Lemma \ref{lem:distal_normal}, we have a normal vector that points away from the base plate.  The ray of points at which the distal center points will be a line emanating from point $\bv D_c$ and in direction $\uv N_D$.  This yields the first formula in the theorem.

% For the second formula, we can use Lemma \ref{lem:distal_normal} to get:
% \begin{align*}
%     \bv D_c + t\uv N_D &= \bv D_c + t (R_M(-\uv z)-\bv D_c)\\
%     &= (1-t)\bv D_c + tR_M(-\uv z)
% \end{align*} 
% \end{proof}

The following lemma gives us a means of connecting the forward-pointing ray to a construction based on the base half-joint which will prove useful in computing inverse kinematics.

\begin{notation}
Let $Z$ denote the $z$-axis, $Z_{\geq 0}$ the non-negative side of the $z$-axis, and $Z_{\leq 0}$ the non-positive side of the $z$-axis in our chosen reference frame.  
\end{notation}

% \CBnote{V3}
\begin{lemma}[Pointing Lemma]\label{lem:pointing}
A configuration of a Canfield joint with midplane $P$ will point at $\bv T$ iff\,\footnote{The word iff is a standard contraction of ``if and only if''.}$R_P(\bv T)\in Z_{\leq 0}$. This holds for backward-pointing by replacing $Z_{\leq 0}$ with $Z_{\geq 0}$.
\end{lemma}

\begin{proof}
Follows from definitions and midplane symmetry. The configuration points at $\bv{T}$ if and only if $\bv{T}$ is contained in its forward-pointing ray. By midpane symmetry, the midplane-reflection of $Z_{\leq 0}$ is the forward-pointing ray (and vice versa). It follows that the midplane reflection $R_P(\bv{T})$ lies in $Z_{\leq 0}$ if and only if $\bv T$ lies on the forward-pointing ray. 

Analogously, midplane symmetry ensures the backward-pointing ray is the midplane reflection of $Z_{\geq 0}$ (and vice versa). Thus, by following an analogous proof we arrive at the Pointing Lemma with `point at' replaced with `backward-point at' and $Z_{\leq 0}$ replaced with $Z_{\geq 0}$.
\end{proof}

\section{Inverse Kinematics}\label{sec:invkin}
Next we handle the inverse kinematics; a means of converting where we want the Canfield joint to point into base angles that achieve the pointing goal. In the context of a 2DoF gimbal, this pointing goal is typically expressed in terms of azimuth and elevation; in some cases, it may be expressed as a target point in XYZ space. In the case of the Canfield joint, however, constraining only the azimuth and elevation  of the distal normal will result in an infinite number of solutions, as shown by the example illustrated in Figure \ref{fig:trivialcase}. Thus, if we want finitely many solutions, we will need appropriate constraints, some of which we explore here and in Section \ref{sec:constrained}.
% constraining by only two degrees of freedom will result in an infinite number of solutions, as shown by the trivial example illustrated in Figure \ref{fig:trivialcase}. 

\begin{figure}[ht]
\centering
\includegraphics[width=.5\textwidth]{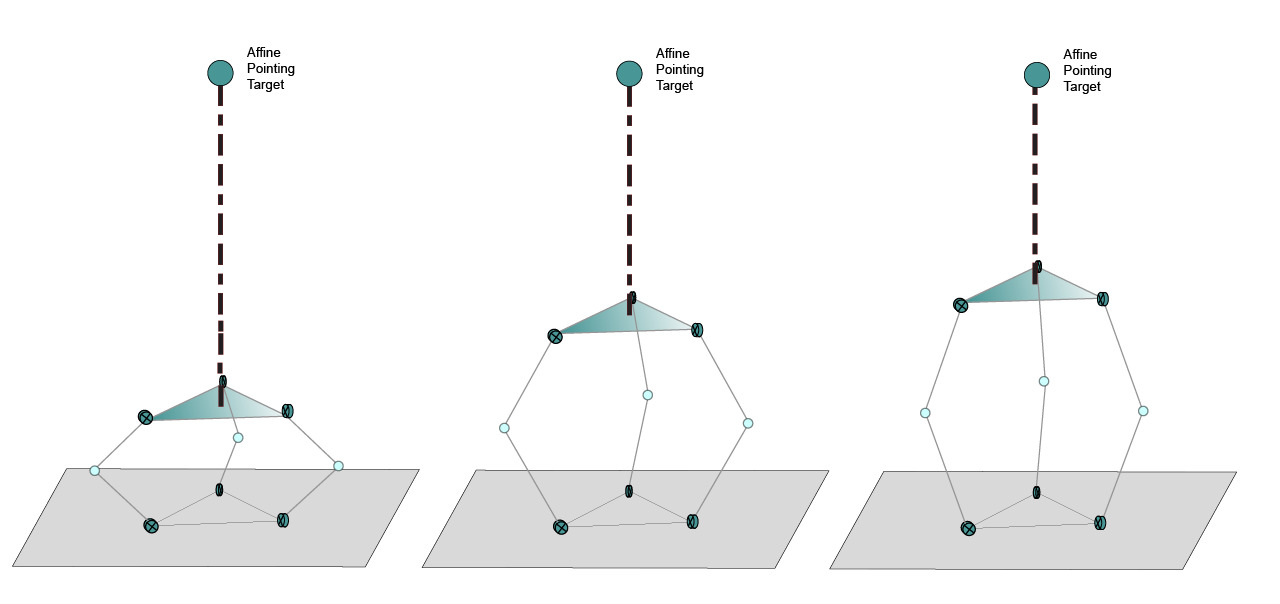}
\caption{Case of multiple configurations with the same pointing result.}
% \CBnote{Make text bigger if necessary}
\label{fig:trivialcase}
\end{figure}

% In order to restrict this set of solutions, we can constrain one degree of freedom. 
How we choose those constraints will change how the inverse kinematic solution is found.  There are three main cases for inverse kinematics, shown in Figure~\ref{fig:inversekinematicsforms}: 

% \begin{figure}
%   \centering
%   \subfigure[]{\includegraphics[width=0.5\textwidth]{Figures/distalcenterpointing.jpg}\label{fig:distalcenterpointing}}
%   \subfigure[]{\includegraphics[width=0.5\textwidth]{Figures/affinepointing.jpg}\label{fig:affinepointing}}
%   \subfigure[]{\includegraphics[width=0.5\textwidth]{Figures/azelcanfield.jpg}\label{fig:azelcanfield}}
%   \caption{Illustrations of inverse kinematics forms. \RSnote{Add subfigure captions.}}
%   \label{fig:inversekinematicsforms}
% \end{figure}

\begin{figure}
\centering
\subfloat[Inverse kinematics may be calculated for a given distal center location...]{\includegraphics[width=0.5\textwidth]{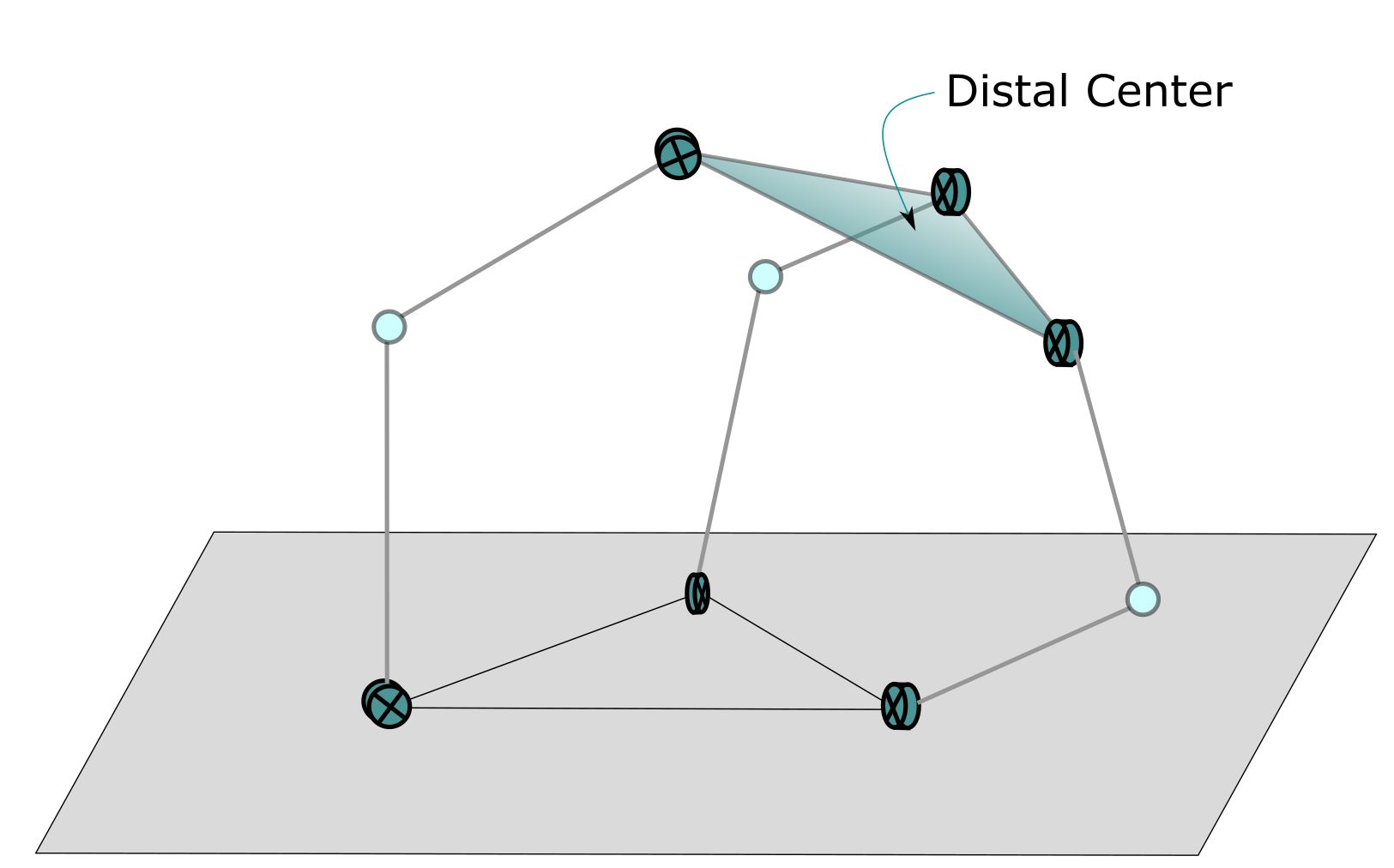}\label{fig:distalcenterpointing}}\hfill
\subfloat[...or a point in 3-space that the configuration should point towards... ] {\includegraphics[width=0.5\textwidth]{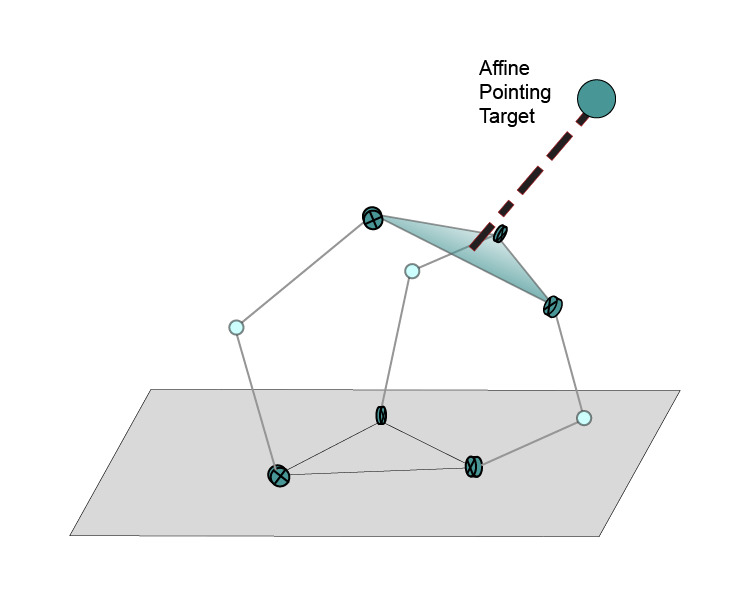}\label{fig:affinepointing}}\hfill
\subfloat[...or a direction in which the distal normal vector should point. NB: The positioning of the depicted configuration does not actually correspond to the distal normal depicted (see Theorem \ref{thm:distalfield}).] {\includegraphics[width=0.5\textwidth]{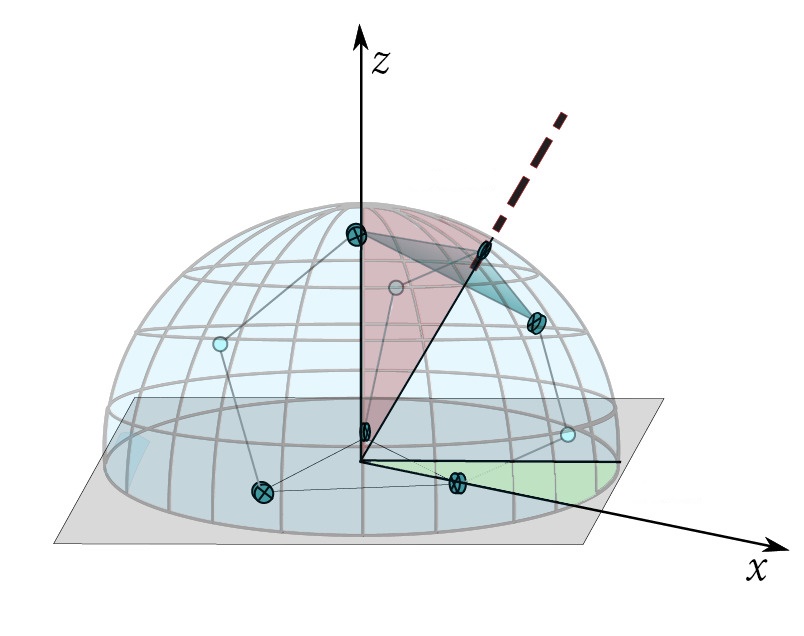}\label{fig:azelcanfield}}\hfill
\caption{Illustrations of inverse kinematics forms.}  \label{fig:inversekinematicsforms}
\end{figure}

\begin{enumerate}
    \item \textbf{Distal Center (3DoF):} A position for the distal center is desired in XYZ space, without regard to orientation of the distal normal vector. This case is straightforward to derive. It may be desirable for a situation with a passively gimbaled end effector, such as in additive manufacturing. This is illustrated in Figure \ref{fig:distalcenterpointing}.
    % \item \textbf{Plunge Distance:} A normal vector for the distal plate is desired, and the plunge distance is restricted to a particular plunge sphere.
    \item \textbf{Affine pointing (3DoF):} The end effector is directed to point at a desired point in XYZ space. This form may be desirable in laser machining, or for directing an instrument to point at a relatively near-field object. This is illustrated in Figure \ref{fig:affinepointing}.
    \item \textbf{Az/El Pointing (2DoF):} The Canfield joint is required to point in a certain direction, i.e. the distal normal has a prescribed orientation (often given in terms of azimuth and elevation, hence the name). This may be used in solar tracking \cite{ntrs1}, or for a survey telescope. This form is illustrated in Figure \ref{fig:azelcanfield}. This provides a good approximation to affine pointing when targets are sufficiently far away.
\end{enumerate}

Each of these three cases ultimately relies upon the inverse kinematic solution to a half-joint. Cases 2) and 3) have analogously defined backward-pointing versions as well. Since we are assuming midplane symmetry, we will be able to determine the solutions to these problems using simple geometric approaches.\footnote{These geometric techniques are sometimes reminiscent of compass and straightedge constructions. As a result, we found them to be readily implementable in related software such as \textit{Geogebra}.}

\tikzstyle{rect} = [rectangle, rounded corners, minimum width=3cm, minimum height=1cm, text width=4cm, text centered, draw=black]
\tikzstyle{rect2} = [rectangle, rounded corners, minimum width=3cm, minimum height=1cm, text width=3cm, text centered, draw=black]
\tikzstyle{arrow} = [thick,->,>=stealth]

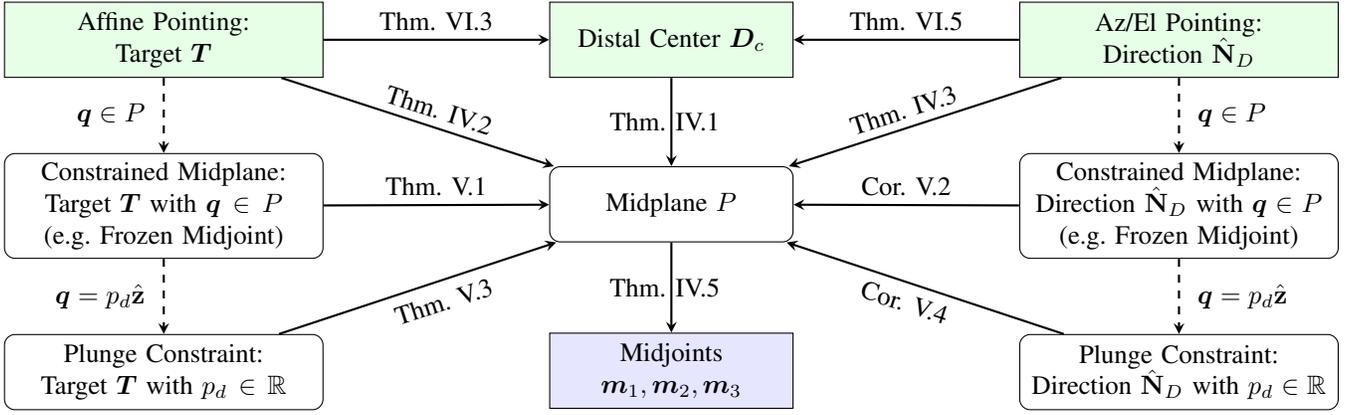
\begin{figure*}[ht]
\centering
\begin{tikzpicture}[node distance=1cm and 3cm]
\node (DC) [rect2, sharp corners, fill=green!10] {Distal Center $\bv D_c$};
\node (MP) [rect2, below=of DC, yshift=-0.175cm] {Midplane $P$};
\node (MJ) [rect2, sharp corners, fill=blue!10, below=of MP, yshift=-0.2cm] {Midjoints \\$\bv m_1,\bv m_2,\bv m_3$};

\node (AP) [rect, sharp corners, fill=green!10, left=of DC] {Affine Pointing:\\ Target $\bv T$};
\node (AP_q) [rect, below=of AP] {Constrained Midplane:\\ Target $\bv T$ with $\bv q\in P$\\(e.g. Frozen Midjoint)};
\node (AP_p) [rect, below=of AP_q] {Plunge Constraint:\\ Target $\bv T$ with $p_d\in\RR$};

\node (AE) [rect, sharp corners, fill=green!10, right=of DC] {Az/El Pointing:\\ Direction $\uv N_D$};
\node (AE_q) [rect, below=of AE] {Constrained Midplane:\\ Direction $\uv N_D$ with $\bv q\in P$\\(e.g. Frozen Midjoint)};
\node (AE_p) [rect, below=of AE_q] {Plunge Constraint:\\ Direction $\uv N_D$ with $p_d\in\RR$};

% % without constraint labels
% \draw [arrow,dashed] (AP) -- (AP_q);
% \draw [arrow,dashed] (AP_q) -- (AP_p);
% \draw [arrow,dashed] (AE) -- (AE_q);
% \draw [arrow,dashed] (AE_q) --(AE_p);

%with constraint labels
\draw [arrow,dashed] (AP) -- node[anchor=east] {$\bv q\in P\,\,$} (AP_q);
\draw [arrow,dashed] (AP_q) -- node[anchor=east] {$\bv q=p_d\uv z\,\,$} (AP_p);
\draw [arrow,dashed] (AE) -- node[anchor=west] {$\,\,\bv q\in P$}(AE_q);
\draw [arrow,dashed] (AE_q) -- node[anchor=west] {$\,\,\bv q=p_d\uv z$}(AE_p);

\draw [arrow] (DC) -- node {Thm. \ref{thm:DC2MP}\phantom{h}} (MP);
\draw [arrow] (MP) -- node {Thm. \ref{thm:MP2MJ}\phantom{h}} (MJ);

\draw [arrow] (AP) -- node[anchor=south] {Thm. \ref{thm:AffineLocusParametrized}} (DC);
\draw [arrow] (AP) -- node[anchor=south west, pos=0.34, yshift=-0.5cm] {\rotatebox{-18}{Thm. \ref{thm:AP2MP}}} (MP);
\draw [arrow] (AP_q) -- node[anchor=south] {Thm. \ref{thm:AP2MP-constrained}} (MP);
\draw [arrow] (AP_p) -- node[anchor=north west, pos=0.375, yshift=0.43cm] {\rotatebox{18}{Thm. \ref{thm:AP2MP-plunge}}} (MP);

\draw [arrow] (AE) -- node[anchor=south] {Thm. \ref{thm:AzElLocusParametrization}} (DC);
\draw [arrow] (AE) -- node[anchor=south east, pos=0.34, yshift=-0.5cm] {\rotatebox{18}{Thm. \ref{thm:AE2MP}}} (MP);
\draw [arrow] (AE_q) -- node[anchor=south] {Cor. \ref{cor:AE2MP-constrained}} (MP);
\draw [arrow] (AE_p) -- node[anchor=north east, pos=0.375, yshift=0.4cm] {\rotatebox{-18}{Cor. \ref{cor:AE2MP-plunge}}} (MP);

\end{tikzpicture}
\vspace*{2.5mm}
\caption{A visual guide to the inverse kinematic. Dashed arrows indicate the introduction of a restriction. Solid arrows indicate inverse kinematics in direction of the arrow. Arrow labels reference the results in this paper corresponding to inverse kinematics of the arrow. Green boxes on the top row are starting positions and the blue box in the center bottom is the terminal position.}
\label{fig:IKflowchart}
\end{figure*}

\subsection{Distal Center to Midplane}\label{sec:DistalCenterIK}

Our first inverse kinematic form begins with a known location for the distal center.  Once the location is known, there is a simple way to construct a midplane.

\begin{theorem}\label{thm:DC2MP}
A Canfield joint configuration has distal center $\bv D_c\neq\bv 0$ iff its midplane is given by
\[
\bv D_c \cdot (\bv x - \tfrac{1}{2}\bv D_c ) = 0.
\]
It has distal center $\bv D_c=\bv 0$ iff its midplane contains the origin.
\end{theorem}

\begin{proof}
By midplane symmetry, $\bv D_c$ is the reflection of the origin ($\bv B_c=\bv 0$) over the midplane.  As such, $\bv D_c - \bv B_c = \bv D_c$ is a normal vector to the midplane when $\bv D_c\neq\bv 0$. In addition, the distance from the origin to the midplane has to be the same as the distance from the midplane to $\bv D_c$.  Thus, a point on the midplane is given by $\tfrac{1}{2}\bv D_c$. Thus, the midplane is $\bv D_c \cdot (\bv x - \tfrac{1}{2}\bv D_c)=0$. Conversely, this plane reflects $\bv B_c$ to $\bv D_c$ by construction.

If $\bv D_c=\bv 0$, then midplane reflection leaves $\bv 0$ unchanged. This can happen only if $\bv 0$ lies on the midplane. Conversely, planes containing $\bv 0$ produces $\bv D_c=\bv 0$ via forward kinematics. 
\end{proof}

\begin{remark}\label{rk:DCis0}
The case where $\bv D_c = \bv 0$ stands out as unusual since any plane passing through $\bv 0$ can be used to construct a configuration of a Canfield joint with that distal center. As such, any distal normal can be achieved when $\bv D_c = \bv 0$. Thus, for any $\bv T$ there is a plane through the origin that points at it.
\end{remark}

% \CBnote{Note sure if we need this part anymore}
% Using this result we can show that any given point can be realized as the distal center of a configuration of some Canfield joint.  Given the proposed distal center and origin we can use Theorem \ref{thm:DistalCentertoMidplane} to find the midplane associated to any viable configuration. Then, once we fix the shape of the base plate, reflecting the base plate over the midplane yields the distal plate.  Connecting each base hinge location to its reflection and then placing a universal joint at the midpoint yields the arms and midjoints. Together this yields a valid Canfield joint and configuration with $\bv D_c = \bv x$.  This means the every point is the distal center of some Canfield joint configuration. The following corollary codifies this result.

% \begin{corollary}
% Given a point $\bv x\in\RR^3$, there exists a Canfield joint with a configuration such that $\bv D_c = \bv x$. 
% \end{corollary}
% \begin{proof}
% We can actually do this using only standard Canfield joints. Let $\bv x$ be our arbitrary point. Let $\bv B_i$ form an equilateral triangle in the base plane centered at the origin. Let theorem
% \end{proof}

\subsection{Affine Pointing to Midplane}

Now we consider the second inverse kinematic form: pointing at a target $\bv T$. As we did for the distal center, we determine the necessary and sufficient criterion to construct midplanes compatible with the objective.  

\begin{theorem}\label{thm:AP2MP}
A configuration with midplane $P$ points at $\bv T$ iff $\bv T\in P\cap Z_{\leq 0}$ or there exists $\bv T\neq\bv K\in Z_{\leq 0}$ such that $P$ can be given by 
\[
    (\bv T-\bv K)\cdot(\bv x - \tfrac{\bv T+\bv K}{2})=0.
\]
This holds for backward-pointing by replacing $Z_{\leq 0}$ with $Z_{\geq 0}$.
\end{theorem}
\begin{proof}
Let $C$ be a configuration with midplane $P$. If $C$ points at $\bv T$, then by the Pointing Lemma \ref{lem:pointing} ${\bv K:=R_P(\bv T)\in Z_{\leq 0}}$. If $\bv T\neq\bv K$, then by symmetry, $P$ orthogonally bisects the line segment $\overline{\bv T\bv K}$. Thus, $P$ can be given by ${(\bv T-\bv K)\cdot(\bv x - \tfrac{\bv T+\bv K}{2})=0}$. If $\bv T=\bv K$ then $R_P(\bv T)=\bv T$ which means $\bv T\in P$. Thus $\bv T\in P\cap Z_{\leq 0}$.

Conversely, suppose $P$ satisfies the theorem's plane equation. Then it orthogonally bisects $\overline{\bv T\bv K}$ for $\bv K\in Z_{\leq 0}$ and so $R_P(\bv T)=\bv K\in Z_{\leq 0}$. Likewise, if $\bv T\in P\cap Z_{\leq 0}$ then $R_P(\bv T)=\bv T\in Z_{\leq 0}$. Thus, either way, $C$ points at $\bv T$ by the Pointing Lemma \ref{lem:pointing}.

The argument for backward-pointing is analogous and uses the backward-pointing version of the Pointing Lemma.
\end{proof}

When $\bv T\notin Z_{\leq 0}$, the above shows us there is a nice way to think of all the solution midplanes for pointing at $\bv T$: they are orthogonal bisectors of line segments  $\overline{\bv T\bv K}$ where $\bv K\in Z_{\leq 0}$.

\subsection{Az/El Pointing to Midplane}

Next we consider the third inverse kinematic form: pointing in direction $\uv N$, i.e. having distal normal $\uv N_D=\uv N$. As before, we solve for the midplane.

\begin{theorem}\label{thm:AE2MP}
A configuration has distal normal
\begin{enumerate}
    \item $\uv N_D\neq -\uv z$ iff its midplane is orthogonal to $\uv N_D+\uv z$,
    \item $\uv N_D= -\uv z$ iff its midplane is parallel to $Z$.
\end{enumerate}
Holds for backward-normals by replacing $\uv N_D$ with $\shortminus\uv N_D$.
\end{theorem}
\begin{proof}
Consider a configuration with midplane $P$ given by $\uv N\cdot(\bv x-\bv q)=0$. By the forward kinematics, the distal normal is given by $\uv N_D = R_{[\uv N,\bv 0]}(-\uv z) = -\uv z+2N_z\uv N$. 
Thus 
\[
    \uv N_D + \uv z= 2N_z\uv N.
\]

Observe that $\uv N_D=-\uv z$ iff $N_z=0$ iff $P$ is parallel to $Z$. On the other hand, $\uv N_D\neq -\uv z$ iff $N_z\neq 0$ iff $\uv N_D+\uv z$ and $\uv N$ are non-zero scalar multiples. Thus $\uv N_D\neq -\uv z$ iff $\uv N_D+\uv z$ is normal to $P$.

The argument for the distal backward-normal is analogous and merely requires replacing $\uv N_D$ with $-\uv N_D$.
\end{proof}

\begin{remark}
We can write $\uv N_D=\sin\theta \guv\rho + \cos\theta \uv z$ where $\guv\rho\perp\uv z$ and $\theta$ is the smallest angle between $\uv N_D$ and $\uv z$ (the polar angle). Applying trigonometric identities we get that
\[
    \fr{\uv N_D+\uv z}{\|\uv N_D+\uv z\|} = \fr{\sin\theta \guv\rho+ (1+\cos\theta) \uv z}{\sqrt{2+2\cos\theta}}= \sin\!\paran{\tfrac{\theta}{2}}\guv\rho + \cos\!\paran{\tfrac{\theta}{2}}\uv z.
\]
Replacing the plane normal $\uv N_D +\uv z$ with $\sin\!\paran{\tfrac{\theta}{2}}\guv\rho + \cos\!\paran{\tfrac{\theta}{2}}\uv z$ in plane equations avoids numerical instability when ${\uv N_D+\uv z\approx \bv 0}$. Similarly useful, we have ${\fr{\uv N_D-\uv z}{\|\uv N_D-\uv z\|} =  \cos\!\paran{\tfrac{\theta}{2}}\guv\rho - \sin\!\paran{\tfrac{\theta}{2}}\uv z}$.
\end{remark}

\subsection{Midplane to Midjoints}\label{sec:MP2MJ}

As we saw in the past three subsections, we can find the associated midplanes for each of our three inverse kinematic forms. However, this does not form a complete solution to the inverse kinematic problem since we ultimately need to arrive at the base settings and not merely a midplane. In the following, we will see how to take a plane $P$ and compute all possible base settings that can realize the plane $P$ as a midplane. Or, if this is not possible, how to determine infeasibility. 

Without loss of generality, we can work with the position of the midjoints instead of base angles since midjoint positions can be readily computed from the base angles and vice versa via trigonometry. For this reason, the Cartesian product of the midcircles $M:=M_1\times M_2\times M_3$ captures all possible base half-joint settings. 

Given a plane $P$ which we wish to realize as a midplane, we will need to find where $P$ intersects each midcircle $M_i$. If $P$ is given by $\bv N \cdot (\bv x - \bv q) = 0$, this means solving for the midjoint $\bv m_i$ in the following system:
\begin{align*}
\bv N\cdot(\bv m_i - \bv q) &= 0, \\
\uv N_i\cdot(\bv m_i - \bv B_i) &= 0,  \\
\norm{\bv m_i - \bv B_i}^2 &= \ell_i^2.
\end{align*}

Because each $M_i$ is contained within its midcircle plane $P_i$, we have that either $P$ is parallel to $P_i$ or not.  The following results deals first with the more difficult non-parallel case and then the simpler parallel cases.

% \begin{lemma}\label{lem:Midplane2MidjointDiscriminant}
% Let $P$ be a plane given by $\bv N \cdot (\bv x - \bv q) = 0$ which is not parallel to the plane containing $M_i$ which is defined by $\uv N_i \cdot (\bv x - \bv B_i)=0$.  Let $\bv q_i$ be some fixed point in the intersection of these two planes and $\bv \delta_i = \bv q_i-\bv B_i$. Let $\bv\nu_i = \bv N \times \uv N_i$ and $\guv\nu_i=\tfrac{\bv\nu_i}{\norm{\bv\nu_i}}$.  Then, there exists a point in $P\cap M_i$ iff
% \[ 
%     \Delta_i := (\hatbv \nu_i \cdot \bv \delta_i )^2 - (\norm{\bv \delta_i}^2 - \ell_i^2) \geq 0. 
% \]
% Moreover, the $i$-th midjoint location(s) are given by 
% \[
%     \bv m_i = \bv q_i - (\hatbv\nu_i\cdot\bv \delta_i \pm \sqrt{\Delta_i}) \hatbv \nu_i. 
% \]
% \end{lemma}

\begin{lemma}\label{lem:Midplane2MidjointDiscriminant}
Let $P$ be a plane given by ${\bv N \cdot (\bv x - \bv q) = 0}$. Suppose $P$ is not parallel to the $i$-th midcircle plane $P_i$ given by $\uv N_i \cdot (\bv x - \bv B_i)=0$.  Let $\bv q_i$ be some fixed point in the intersection of these two planes and $\bv \delta_i = \bv q_i-\bv B_i$. Let ${\bv\nu_i = \bv N \times \uv N_i}$ and $\guv\nu_i=\tfrac{\bv\nu_i}{\norm{\bv\nu_i}}$.  Then, there exists a point in $P\cap M_i$ iff
\[ 
    \Delta_i := (\hatbv \nu_i \cdot \bv \delta_i )^2 - (\norm{\bv \delta_i}^2 - \ell_i^2) \geq 0. 
\]
Moreover, the $i$-th midjoint location(s) are given by 
\[
    \bv m_i = \bv q_i - (\hatbv\nu_i\cdot\bv \delta_i \pm \sqrt{\Delta_i}) \hatbv \nu_i. 
\]
If instead $P$ is parallel to $P_i$ but $P\neq P_i$, then $P\cap M_i=\varnothing$. If $P=P_i$, then $P\cap M_i=M_i$.
\end{lemma}
\begin{proof}
Suppose $P$ is not parallel to $P_i$ and let $L_i$ be their line of intersection. Since $\bv q_i\in L_i$ and $\hatbv \nu_i$ is a vector oriented parallel to $L_i$, we can parametrize $L_i$ with the formula
\[
    L_i(t) := \bv q_i + t\hatbv \nu_i, \quad t\in\mathbb{R}.
\]

If $P\cap M_i$ is nonempty, then there must be some $t$ for which $L_i(t) \in M_i$.  Since all points of $M_i$ are a distance $\ell_i$ from $B_i$ this means
\begin{align*}
    \ell_i^2 &= \norm{\bv q_i + t\hatbv \nu_i - \bv B_i}^2\\
             &= \norm{\bv q_i - \bv B_i + t\hatbv\nu_i }^2\\
             &= \norm{\bv q_i - \bv B_i}^2 + \norm{\hatbv \nu_i}^2 t^2 + 2t\, \hatbv \nu_i \cdot (\bv q_i - \bv B_i).
\end{align*}
Since $\bv\delta_i = \bv q_i - \bv B_i$ and $\norm{\hatbv\nu_i}=1$ we have
\[
    t^2 + 2(\hatbv \nu_i \cdot \bv \delta_i ) t + (\norm{\bv \delta_i}^2 - \ell_i^2) = 0.
\]
Using the quadratic formula to solve for $t$ yields
\[
    t = -(\hatbv{\nu_i}\cdot\bv\delta_i) \pm \sqrt{(\hatbv{\nu_i}\cdot\bv\delta_i)^2-(\norm{\bv\delta_i}^2-\ell_i^2)}.
\]

Thus, there will be solutions iff the discriminant $\Delta_i \geq 0$, and any member of $P\cap M_i$ will have the form 
\[ 
\bv q_i - (\hatbv\nu_i\cdot\bv \delta_i \pm \sqrt{\Delta_i}) \hatbv \nu_i. 
\]
Lastly, we consider what happens if $P$ is parallel to $P_i$. If ${P\neq P_i}$, then $P$ and $P_i$ are disjoint so $P\cap M_i=\varnothing$. If $P=P_i$, then $M_i$ is contained in $P$ so $P\cap M_i=M_i$.
\end{proof}

% % 2nd attempt at lemma
% \CBnote{define midplane of a half-joint to be a plane that passes through all three midjoints}
\begin{remark}\label{rk:qi}
We can always choose $\bv q_i$ to be the orthogonal projection of $\bv q$ onto $L_i$, in which case it can be given explicitly as $\bv q_i = \bv q - \tfrac{\uv N_i \cdot (\bv q - \bv B_i)}{\uv N_i\cdot(\bv N\times \hatbv \nu_i)}\bv N\times \hatbv\nu_i$. It is immediate that $\bv q_i\in L_i$ because $\bv N\cdot(\bv q_i - \bv q)=0$ and $\uv N_i\cdot(\bv q_i - \bv B_i)= 0$. Also $\hatbv\nu_i\cdot(\bv q_i - \bv q)=0$ and that finishes the argument that $\bv q_i$ is the orthogonal projection of $\bv q$.
\end{remark}

From here, intersecting the candidate midplane with each midcircle will yield the compatible midjoint locations that could support the plane.  Combining these intersections gives the full solution set.

\begin{theorem} \label{thm:MP2MJ}% Wouldn't label this one midplane to midjoint. Maybe Solution Cardinality
Given a Canfield joint and a plane $P$, the set of all midjoint positions $(\bv m_1,\bv m_2,\bv m_3)$ where each $\bv m_i \in P$ is given by
\[
     (M_1\cap P)\times(M_2\cap P)\times(M_3\cap P).
\]
Each factor of the Cartesian product can be computed via Lemma $\ref{lem:Midplane2MidjointDiscriminant}$ and the following are the only possibilities:
\begin{enumerate}
    \item $P$ intersects every midcircle and is not parallel to any midcircle plane. In this case the solution set is finite.
    \item $P$ intersects every midcircle but $P$ is identical to some midcircle plane.  In this case, the solution set is infinite.
    \item $P$ does not intersect one of the midcircles.  In this case, the solution set is empty.
\end{enumerate}
\end{theorem}

\begin{proof}
By definition of intersection, the set of midjoint position tuples is given by $(M_1 \cap P) \times (M_2 \cap P) \times (M_3 \cap P)$. 

If we have $P\cap M_i\neq\varnothing$ and $P$ is not parallel to $P_i$, then $P\cap M_i$ has at most two points by Lemma \ref{lem:Midplane2MidjointDiscriminant}.  If this happens for all $i$ then the solution set is a Cartesian product of three sets of size at most 2, and is therefore a set of size at most 8.

If $P$ is identical to some $P_i$, then $P\cap M_i=M_i$ which is infinite. If $P$ intersects the other two midcircles, then the Cartesian product is a product of two nonempty sets and an infinite set. Thus, the solution set is infinite.

Lastly, if $P\cap M_i= \varnothing$ for some $i$, then we have a Cartesian product with one empty factor.  Thus, the solution set is empty.
\end{proof}

Even in case (1), there may be multiple solutions.  In fact, case (1) can yield up to 8 different midjoint settings by reflecting $\bv m_i$ across the line between $\bv B_i$ and $\bv D_i$ (Figure \ref{fig:multiplesolns}).  However, for both (1) and (2), physical restrictions can limit the number of solutions by making some points of $P\cap M_i$ unobtainable. It is up to the reader to incorporate their physical constraints appropriately and find a solution that suits their needs.
% The solutions to these equations are frequently non-unique.  As an example, see Figure \ref{fig:multiplesolns}.  In some cases, physical restrictions can help limit the number of solutions, but sometimes a choice will need to be made by the controller.  Since our inverse kinematic output is generally going to be a proposed midplane, we characterize the possible solution spaces with the following proposition.\RSnote{Move this discussion and figure to around Algorithm}

% \begin{figure}[ht]
% \centering
% \includegraphics[width=0.5\textwidth]{Figures/midplaneintersect.JPG}
% \caption{The midplane intersects each midcircle at two points. In this case, there are eight possible solutions to the inverse kinematics - although, in a practical application, several of these will likely not be viable. \KCnote{FIG: Replace notional figure midplaneintersect.jpg.}}
% \end{figure}

\begin{figure}[ht]
\centering
\includegraphics[width=.5\textwidth]{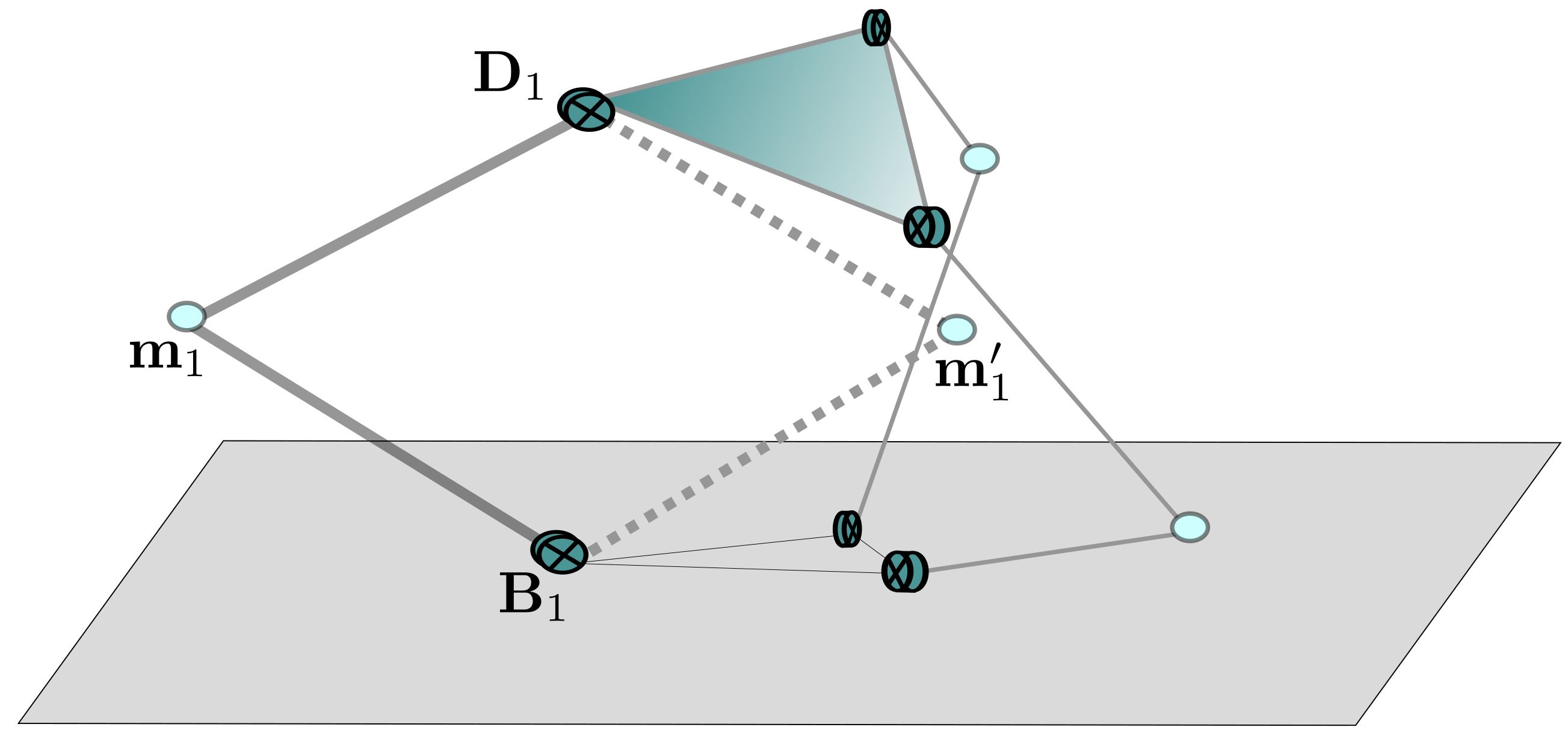}
\caption{\label{fig:multiplesolns}For a given midplane, there may be 1, 2, 4, or 8 inverse kinematic solutions. Here, multiple solutions for one leg are shown.}
\end{figure}

% \begin{proposition} \label{BaseAngleAmbiguity} \RSnote{Tri-Fix Thm to become Free}
% Given a Canfield joint with parameters $\ell$ and $b$ and a plane in XYZ space, exactly one of the following is true:
% \begin{enumerate}
%     \item There are fewer than three midcircles that intersect the plane.
%     \item The plane intersects each midcircle at a single point.
%     \item The plane intersects one midcircle at two points, and the other two at a single point each.
%     \item The plane intersects one midcircle at a single point, and the other two at two points each.
%     \item The plane intersects each midcircle at two points.
% \end{enumerate}
% \end{proposition}

% \RSnote{Q: Do we need to prove this proposition?}\CBnote{A: I think this and previous theorem can be merged into one theorem by using the explicit solution I posted in the group chat}

% For each of these cases, there are a different number of solutions for the equations in Theorem \ref{thm:MidplaneToBaseAngleEqn}.  Specifically: Case 1 yields no solutions, Case 2 yields a unique solution, Case 3 yields two solutions, Case 4 yields four solutions, and Case 5 yields eight solutions.  An implementation of these results should take care that the choice of planes made in any of the inverse kinematics sections do not fall into Case 1.

\section{Constrained Kinematics: \\Frozen Midjoints and Plunge Distance}\label{sec:constrained}
% \RSnote{Seized midjoint vs Seized base hinge distinction}
% In \cite{BobTopCJPaper}, a failure mode for a Canfield joint was discussed in which one of the midjoints is locked in place due to a motor at the base hinge seizing up. We refer to this as a frozen midjoint configuration.  \RSedit{That paper addresses the workspace in a frozen midjoint configuration and shows that i}{I}t is still possible to point in a wide range of directions even in this failure mode.  \RSedit{As such, the inverse kinematics within this failure mode can be treated significantly differently than in the classical case.  In essence, this is because t}{T}he Canfield joint loses a degree of freedom since one of the midjoints is frozen.  \RSedit{Seeing this play out mathematically is a significant challenge that we address here.  As a reminder, we are still assuming that our Canfield joint configurations have midplane symmetry, which allows us to work geometrically.}{}%\RSnote{WRITE:Compare Canfield approach to this approach. Note that restricting plunge distance unnecessarily (and severely) restricts workspace in this case}

In this section we consider the problem of pointing under additional constraints. Specifically, we will be constraining our Canfield joint by requiring that its midplanes always contain some prescribed fixed point $\bv q$. At first, this may seem artificial, but this arises naturally in two  situations. The first is the popular plunge distance approach for controlling the Canfield joint in a gimbal-like manner. The second is a failure mode in which one of the midjoints is locked in place relative to the base plate due to its base hinge seizing up.

The former is well-known and introduced in Canfield's thesis \cite{Canfield1998}. The failure mode is less studied but has been discussed in \cite{BobTopCJPaper} from a different perspective. It is particularly important to study in the context of deep space applications since repair may be impossible. One lesson from \cite{BobTopCJPaper} is that despite the loss of a degree of freedom, it is still possible to point in a wide range of directions within the failure mode. However, in \cite{BobTopCJPaper}, we also see that maintaining a fixed plunge distance within the frozen midjoint failure mode unnecessarily limits the range of motion. As such, inverse kinematic solutions that do not depend on plunge distance are valuable in controlling the Canfield joint stuck in this failure mode.

In this section we shall provide exactly such a solution to the inverse kinematics of the frozen midjoint. We will also generalize the notion of plunge distance and extend the known inverse kinematics to these cases. Furthermore, we will unify the frozen midjoint and generalized plunge distance kinematics into one framework.

\subsection{Constrained Midplane Inverse Kinematics}

The following two results solve the affine and Az/El inverse kinematic problems for Canfield joints with midplanes constrained to contain a prescribed point $\bv q$.

\begin{notation}
We use $S_{(\bv p,\bv q)}$ to denote the sphere centered at $\bv p $ that passes through point $\bv q$ (where sphere means a ball's surface).
\end{notation}

\begin{theorem}\label{thm:AP2MP-constrained}
Consider a Canfield joint constrained to have a fixed point $\bv q$ in its midplane. Let $C$ be a configuration of this constrained Canfield joint, $P$ its midplane, and $\bv T$ some target point. Then the following hold:
\begin{enumerate}[a)]
    %\item $\bv T\in Z_{\leq 0}$ and $P$ contains both $\bv T$ and $\bv q$ or,
    \item If $\bv T\notin Z_{\leq 0}$, then $C$ points at $\bv T$ iff there exists a point $\bv K\in S_{(\bv q,\bv T)}\cap Z_{\leq 0}$ so that $P$ can be given by
    \[
        (\bv T-\bv K)\cdot(\bv x -\bv q)=0.
    \]
    \item If $\bv T\in Z_{\leq 0}$, then $C$ points at $\bv T$ iff $P$ contains both $\bv q$ and $\bv T$, or $P$ is given by $\uv z\cdot(\bv x - \bv q)=0$ where $q_z\leq \tfrac{1}{2}T_z$.
\end{enumerate}
For backward-pointing, replace every `$\leq$' with `$\geq$'.
\end{theorem}
\begin{proof}
Due to the midplane constraint, we can assume all midplanes have the form $\bv N \cdot (\bv x - \bv q)=0$ for some $\bv N\neq\bv 0$.
\begin{enumerate}[a)]
    \item Consider $\bv T\notin Z_{\leq 0}$. Assume there is a $\bv K \in S_{(\bv q,\bv T)}\cap Z_{\leq 0}$.  Note that $\bv T \neq \bv K$ since $\bv T \notin Z_{\leq 0}$ and that $\bv T \in S_{(\bv q,\bv T)}$. Let $P$ be the plane given by $(\bv T - \bv K)\cdot(\bv x - \bv q)=0$. Since $\bv T-\bv K\neq \bv 0$ is orthogonal to $P$ and both $\bv T,\bv K\in S_{(\bv q,\bv T)}$, it follows that $R_P(\bv T) = \bv K$.  Since $\bv K \in Z_{\leq 0}$, by the Pointing Lemma \ref{lem:pointing}, we have that $C$ points at $\bv T$.
    
    \quad Conversely, if $C$ points at $\bv T$, then by the Pointing Lemma \ref{lem:pointing}, $R_P(\bv T)\in Z_{\leq 0}$.  Let $\bv K=R_P(\bv T)$ and note that $\bv T-\bv K\neq \bv 0$ is orthogonal to $P$. Additionally, $\norm{\bv q-\bv T} = \norm{\bv q-\bv K}$ since $\bv q\in P$. Thus $\bv K\in S_{(\bv q,\bv T)}$ and so $\bv K\in S_{(\bv q, \bv T)}\cap Z_{\leq 0}$. Thus $P$ is given by ${(\bv T - \bv K)\cdot (\bv x - \bv q) =0}$.
    
    \item Consider $\bv T\in Z_{\leq 0}$. If $\bv q,\bv T\in P$, then $R_P(\bv T)=\bv T\in Z_{\leq 0}$. If $P$ is given by $\uv z\cdot(\bv x - \bv q)=0$ where $q_z\leq\tfrac{1}{2}T_z$, then since $\bv T=T_z\uv z$ we have
    \begin{align*}
        R_P(\bv T)=R_{[\uv z,\bv q]}(\bv T) = T_z\uv z - 2(T_z-q_z)\uv z = (2q_z-T_z)\uv z.
    \end{align*}
    Since $q_z\leq\tfrac{1}{2}T_z$, we have $R_P(\bv T)\in Z_{\leq 0}$. Thus, either way, $C$ points at $\bv T$ by the Pointing Lemma \ref{lem:pointing}.
    
    \quad Conversely, if $C$ points at $\bv T$, then $R_P(\bv T)\in Z_{\leq 0}$ by the Pointing Lemma \ref{lem:pointing}. If $\bv T\notin P$, then $\bv T\neq R_P(\bv T)$. Since both $\bv T, R_P(\bv T)\in Z_{\leq 0}$, it follows that $\uv z$ is normal to $P$. Thus, $P$ must be given by $\uv z\cdot(\bv x - \bv q)=0$. By our above calculation, $R_P(\bv T)\in Z_{\leq 0}$ only if $q_z\leq\tfrac{1}{2}T_z$. On the other hand, if $\bv T\in P$ then $P$ contains both $\bv q$ and $\bv T$, as desired.
\end{enumerate}
The backward-pointing version of the Pointing Lemma \ref{lem:pointing} yields the analogous result which has every `$\leq$' replaced with `$\geq$'.
\end{proof}

Note that if $\bv T\in Z_{\leq 0}$ and there exists a different point $\bv K\in S_{(\bv q,\bv T)}\cap Z_{\leq 0}$, then $q_z\leq\tfrac{1}{2}T_z$ and $\bv T-\bv K$ is a multiple of $\uv z$. Thus the procedure in a) correctly produces the plane ${\uv z\cdot(\bv x-\bv q)=0}$ from b). However, we can also choose ${\bv K=\bv T}$ as our point of intersection in $S_{(\bv q,\bv T)}\cap Z_{\leq 0}$. In this case ${\bv T-\bv K=0}$, and so the provided plane equation breaks down. However, from b) we know that there are other solutions, and that they can be any plane containing both $\bv T$ and $\bv q$. Thus when ${\bv T\in Z_{\leq 0}}$, we have \textit{infinitely} many new solutions not captured by the technique in a). For this reason we broke this theorem into two cases: $\bv T\notin Z_{\leq 0}$ and $\bv T\in Z_{\leq 0}$. The first case being the most general, useful, and simplest to implement.\footnote{Note that to construct the plane in Theorem \ref{thm:AP2MP-constrained}a) we intersect spheres and rays, and then construct a plane through $\bv q$ which is orthogonal to the line segment $\overline{\bv T\bv K}$. All these steps are very geometric and have a 3D compass and straight-edge flavor. Indeed, this result can be readily implemented in geometric software such as \textit{GeoGebra}.}

\begin{corollary}\label{cor:AE2MP-constrained}
Consider a configuration with midplane $P$ containing $\bv q$ and with distal normal $\uv N_D$. Then
\begin{enumerate}
    \item $\uv N_D\neq -\uv z$ iff $P$ can be given by $(\uv z + \uv N_D)\cdot(\bv x -\bv q)=0$.
    \item $\uv N_D=-\uv z$ iff $P$ is parallel to $Z$ and $\bv q\in P$.
\end{enumerate}
Holds for backward-normal by replacing $\uv N_D$ with $\shortminus\uv N_D$.
\end{corollary}
\begin{proof}
This follows immediately from Theorem \ref{thm:AE2MP} by adding the constraint that $\bv q\in P$.
\end{proof}

%%%%%%%%%%%%%%%%%%%%%%%%%%%%%%%%%%%%%%%%%%%%%%%%%%%%%%%%%%%%%%%%%%%%%%%%%%%%%%%%%%%%
\subsection{Pointing with a Frozen Midjoint}

% \begin{figure}[h]
% \begin{tikzpicture}
% \begin{scope}[xshift=0cm]
%     \node[anchor=center,inner sep=0] (image) at (0,0) {\includegraphics[width=0.5\textwidth]{Figures/DistalNormal_FrozenFeasibility_Lis2.pdf}};
% \end{scope}
% \end{tikzpicture}
% \caption{Standard Canfield joint with $b=\sqrt{3}$ and $\ell=2$. Top row: Feasible distal normal orientations (accounting for base half-joint geometry). Bottom row: Position of frozen midjoint. }
% \label{fig:feaspointing}
% \end{figure}

\begin{figure}[ht]
\includegraphics[trim={0.5cm 2.5cm 0 0},clip,width=0.5\textwidth]{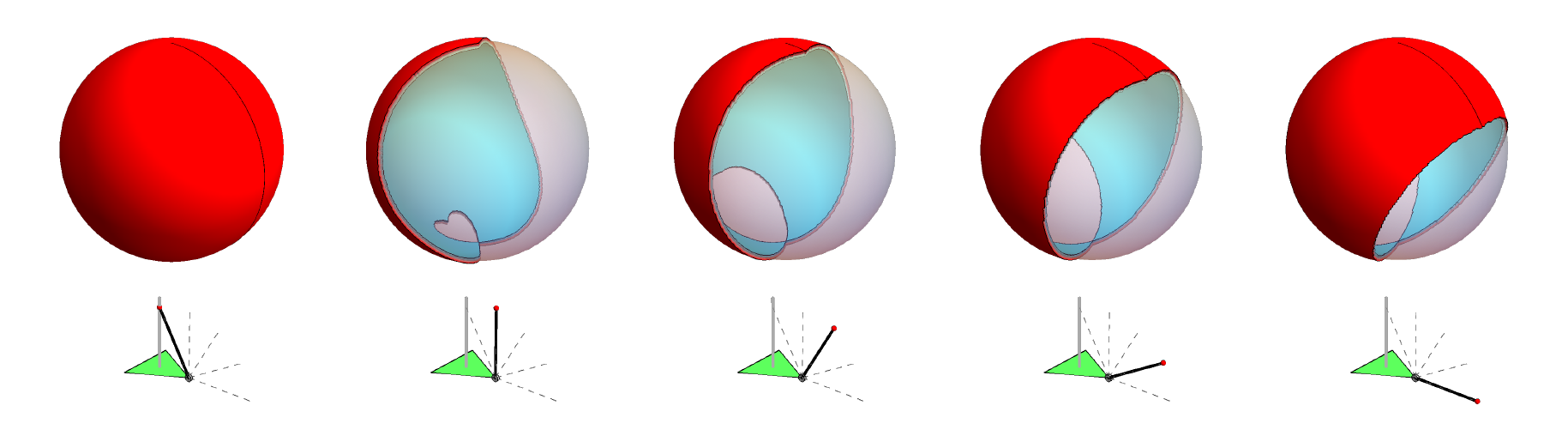}
\caption{A Standard Canfield joint with $b=\sqrt{3}$ and $\ell=2$. Each column assumes a frozen midjoint and is analyzed as in Section \ref{sec:feasibleDN}. Bottom row: Position of frozen midjoint. Top row: Feasible distal normal orientations due to the frozen midjoint position (allowing self-intersection of the mechanism). The feasible region is painted red on the exterior surface and cyan on the interior surface; the unreachable part of the unit sphere is translucent. }
\label{fig:feaspointing_frozen}
\end{figure}

Unless otherwise noted, we will assume that $\bv m_1$ is the frozen midjoint and denote it by $\bv{m}^*=\bv m_1$ since relabeling midjoints will yield analogous results. Freezing a midjoint at $\bv m^*$ implies only considering midplanes that contain $\bv m^*$. Thus, the kinematics in this case is precisely that which we described in the previous subsection but with $\bv q=\bv m^*$. 

As a result, Theorem \ref{thm:AP2MP-constrained} and Corollary \ref{cor:AE2MP-constrained} for $\bv q=\bv m^*$ solve pointing for this failure mode so long as we know the location of the frozen midjoint $\bv m^*$. In practice, this amounts to knowing the base angle for the frozen midjoint, which is something that the user knows or can likely infer. 

However, these theorems do not account for the geometry of the Canfield joint which can make some of the solutions geometrically impossible. To account for that, we would need to check whether the solution midplane intersects each midcircle. Using the result in Section \ref{sec:MP2MJ} we can answer these questions and in Section \ref{sec:feasibleDN} we show how to determine which distal normal orientations can be obtained (accounting for the the joint geometry). Figure \ref{fig:feaspointing_frozen} illustrates the pointing capabilities within 5 failure mode scenarios. Since distal normals are always unit vectors, they have a one-to-one correspondence to points on a unit sphere. For a given frozen midjoint configuration, shown in the second row, the set of distal normals which can be achieved are opaque on the sphere above it. Here we can see that there is still a large amount of accessible distal normal orientations despite having one midjoint locked in place. In Section \ref{sec:frozenplunge} we will explain the extraordinary pointing capability exhibited in the leftmost example, where the set of achievable distal normals encompasses the entire sphere.

When applying Theorem \ref{thm:AP2MP-constrained}a) we need to check that the $\bv K$ we find lies in $Z_{\leq 0}$. Since a sphere intersects a line in either 0, 1, or 2 locations, there can be up to 2 candidate midplanes in this case. It is straightforward to compute what $\bv K$ is; we simply need to find the points on the sphere ${\norm{\bv K-\bv m^*}^2=\norm{\bv m^*-\bv T}^2}$ with $K_x=K_y=0$ and $K_z\leq 0$. Keeping mind that $K_z\leq 0$, this simplifies to solving
\[
    {m^*_x}^2 + {m^*_y}^2 + (K_z-m^*_z)^2 = \norm{\bv m^*-\bv T}^2,
\]
and so by the quadratic equation we have
\[
    K_z = m_z^*\pm\sqrt{\norm{\bv m^*-\bv T}^2-({m^*_x}^2+{m^*_y}^2)}.
\]
We need to determine which (if any) of these two solutions yield $K_z\leq 0$. 

The above instantly tells us that there is no way to point at $\bv T$ whenever $\norm{\bv m^*-\bv T}^2<{m^*_x}^2+{m^*_y}^2$. Checking for when $K_z\leq 0$ introduces a further constraint. Thus the conditions for $\bv T$ that need to be satisfied for the existence of $\bv K\in Z_{\leq 0}$ are 
\begin{alignat*}{2}
    \norm{\bv m^*-\bv T}^2 &\geq {m^*_x}^2+{m^*_y}^2,\qquad && m^*_z\leq 0,\\
    \norm{\bv m^*-\bv T}^2 &\geq \norm{\bv m^*}^2,\qquad && m^*_z\geq 0
\end{alignat*}
% Thus if $m_z^*\leq 0$ for there to exist a pointing solution then $\bv T$ needs to satisfy
% \[
%     \norm{\bv m^*-\bv T}^2 \geq {m^*_x}^2+{m^*_y}^2
% \]
% and if $m_z^*\geq 0$ then $\bv T$ needs to satisfy
% \[
%     \norm{\bv m^*-\bv T}^2 \geq \norm{\bv m^*}^2
% \]
Finally, once a midplane is determined, we can use Theorem \ref{thm:MP2MJ} to find the midjoint positions for the non-frozen midjoints if they exist. For the Az/El problem, things are simpler and we just apply Corollary \ref{cor:AE2MP-constrained} and Theorem \ref{thm:MP2MJ} directly.

%%%%%%%%%%%%%%%%%%%%%%%%%%%%%%%%%%%%%%%%%%%%%%%%%%%%%%%%%%%%%%%%%%%%%%%%%%%%%%%%%%%%
\subsection{Pointing with a Fixed Plunge Distance}

\begin{figure}[ht]
\centering
\includegraphics[width=0.5\textwidth]{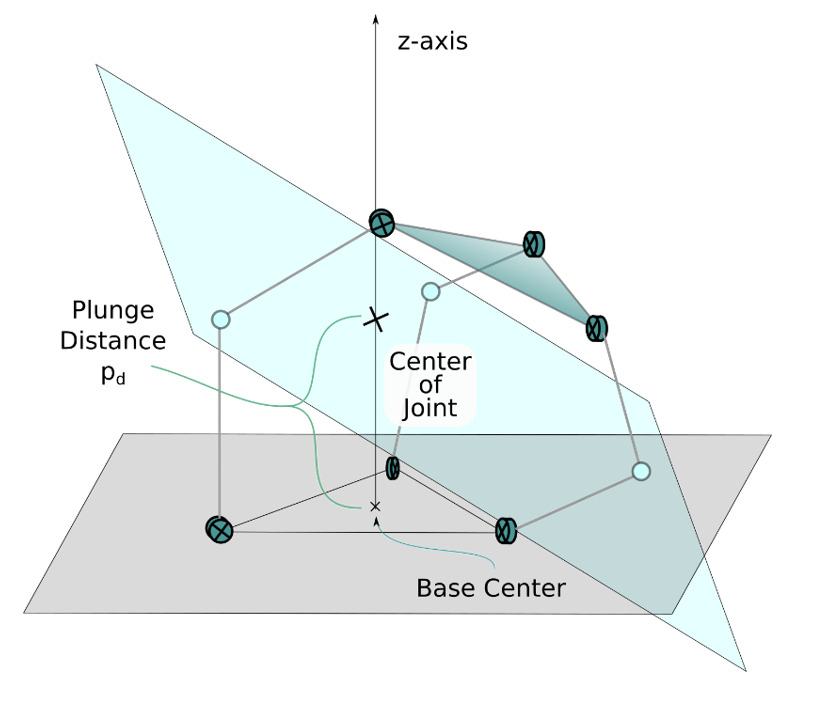}
\caption{Illustration of plunge distance.}
\end{figure}

In \cite{Canfield1998}, the idea of the plunge distance is used to constrain the kinematics so that a standard Canfield joint can be effectively controlled.  As such, we include this here for completeness and to provide our own more general definitions and results.
% We allow for zero, negative, and infinite plunge distance, which are meaningfully different from positive plunge distance in terms of pointing the Canfield joint and not considered in \cite{Canfield1998}. 

\begin{definition}
Given a Canfield joint configuration, we say it has \textbf{plunge distance} $p_d\in\RR$ if the midplane intersects the $Z$ axis at the point $\bv c_J=p_d\uv z$. If the midplane does not intersect the $Z$ axis, then we say $p_d = \infty$ (see Figure \ref{fig:vertmidplane}).\footnote{There is no distinction between $-\infty$ and $\infty$. We choose this because a midplane which doesn't intersect $Z$ in $\RR^3$ intersects $Z$ at exactly one ``point at infinity'' in the projective geometry sense.} If the Canfield joint is constrained to have plunge distance $p_d$, we call $\bv c_J$ the \textbf{center} of the Canfield joint. 
\end{definition}
\begin{remark}
It is possible for the midplane to completely contain the $Z$ axis. In these cases, the above definition implies every finite plunge distance is achieved simultaneously.
\end{remark}

This definition naturally extends the one found in \cite{Canfield1998} by working for arbitrary choice of base center and allowing for zero, negative, and infinite plunge distances. As a result, this captures every way the midplane can intersect the $Z$ axis. It is worth noting that plunge distance can potentially take on values outside the conventional range of ${0<p_d\leq \ell}$. This range has been embraced in past prototypes because it is well-suited to upward hemispherical pointing, but extending the plunge distance enables downward pointing with sufficiently long arms, as shown in Figure \ref{fig:feaspointing}.

Next we show how to obtain the midplane for the affine pointing problem under a finite plunge distance constraint. In principle, we could simply apply Theorem \ref{thm:AP2MP-constrained} with $\bv q=p_d\uv z$ and get our solution. However, knowing the midplane intersects $Z$ is a very special case and it will allow us to obtain very sharp result with extra effort. The following result handles every case except infinite plunge distance. 
%Both of these propositions are wrong.
% \begin{proposition}
% Given a Canfield joint with parameters $\ell$ and $b$ with midplane symmetry where the distal center is such that $D_z >0$, then $p_d < \infty$.
% \end{proposition}

% \RSnote{MATH:  Proof of Prop: Existence of Plunge Distance}

% \begin{proposition} \label{prop:PlungeDistBounds}
% Given a Canfield joint with parameters $\ell$ and $b$, the plunge distance $p_d$ satisfies:
% \[ 0 \leq p_d \leq 2\ell \]
% \end{proposition}
% \RSnote{MATH:  Prop: Proof of Plunge Bounds}

\begin{theorem}\label{thm:AP2MP-plunge}
Consider a Canfield joint constrained to have plunge distance $p_d\in\RR$. Let $C$ be a configuration of this constrained Canfield joint, $P$ its midplane, and $\bv T$ some target point. Then the following hold:
\begin{enumerate}[a)]
    \item If $\bv T\notin Z_{\leq 0}$ and $\norm{\bv T-\bv c_J}\geq\abs{p_d}$, then $C$ points at $\bv T$ iff $P$ is given by 
    \[
        \paran{\uv z + \tfrac{\bv T-\bv c_J}{\norm{\bv T-\bv c_J}}}\cdot(\bv x -\bv c_J)=0
    \]
    \item If $\bv T\notin Z_{\leq 0}$ and $\norm{\bv T-\bv c_J}\leq\abs{p_d}$, then $C$ points at $\bv T$ iff and $p_d\leq 0$ and $P$ is given by one of
    \[
        \paran{\uv z \pm \tfrac{\bv T-\bv c_J}{\norm{\bv T-\bv c_J}}}\cdot(\bv x -\bv c_J)=0
    \]
    \item If $\bv T\in Z_{\leq 0}$, then $C$ points at $\bv T$ iff $P$ contains both $\bv c_J$ and $\bv T$, or $P$ is given by $\uv z\cdot(\bv x - \bv c_J)=0$ where $p_d\leq\tfrac{1}{2}T_z$.
\end{enumerate}
These hold for backward-pointing by replacing `$\leq$' with `$\geq$' wherever $Z_{\leq 0}$, $p_d\leq 0$, and $p_d\leq \tfrac{1}{2}T_z$ occur.
\end{theorem}
\begin{proof}
By the plunge constraint, we can assume all midplanes have the form $\bv N \cdot (\bv x - \bv c_J)=0$ for some $\bv N\neq\bv 0$ and $\bv c_J=p_d\uv z$. 

First assume $\bv T\in Z_{\leq 0}$. Part c) follows immediately by using Theorem \ref{thm:AP2MP-constrained} for $\bv q=\bv c_J$ (and hence $q_z=p_d$).

Next assume $\bv T\notin Z_{\leq 0}$. Using Theorem \ref{thm:AP2MP-constrained} for $\bv q=\bv c_J$, we know that $C$ will point at $\bv T$ iff there is a $\bv K\in S_{(\bv c_J,\bv T)}\cap Z_{\leq 0}$ so that $P$ can be given by $(\bv T-\bv K)\cdot(\bv x-\bv c_J)=0$. Note that $\bv K\in S_{(\bv c_J,\bv T)}\cap Z$ iff $\bv K=K_z\uv z$ and
\[
\norm{\bv T-\bv c_J}^2=\norm{\bv K-\bv c_J}^2=(K_z-p_d)^2.
\]
Solving for $K_z$ via the quadratic equation means
\[
    K_z = p_d \pm \norm{\bv T-\bv c_J}.
\]
Thus, $C$ points at $\bv T\notin Z_{\leq 0}$ iff $P$ can be given by 
\[
    (\bv T-(p_d\pm\norm{\bv T-\bv c_J})\uv z)\cdot(\bv x-\bv c_J)=0,
\]
and $p_d\pm\norm{\bv T-\bv c_J}\leq 0$. Using $\bv c_J=p_d\uv z$ and dividing both sides $\norm{\bv T-\bv c_J}$ this plane equation can be rewritten as
\[
    \paran{\tfrac{\bv T- \bv c_J}{\norm{\bv T-\bv c_J}}\mp\uv z}\cdot(\bv x-\bv c_J)=0.
\]
Now we tackle a) and b):
\begin{enumerate}[a)]
    \item If we additionally assume $\norm{\bv T-\bv c_J}\geq\abs{p_d}$, then $-\norm{\bv T-\bv c_J}\leq p_d\leq \norm{\bv T-\bv c_J}$. From this we get $p_d-\norm{\bv T-\bv c_J}\leq 0$ and $0\leq p_d+\norm{\bv T-\bv c_J}$, meaning $K_z=p_d-\norm{\bv T-\bv c_J}$ is the only option. Thus, under this additional assumption, $C$ points at $\bv T$ iff $P$ is given by
    \[
        \paran{\tfrac{\bv T- \bv c_J}{\norm{\bv T-\bv c_J}}+\uv z}\cdot(\bv x-\bv c_J)=0.
    \] 
    \item If instead we additionally assume $\norm{\bv T-\bv c_J}\leq\abs{p_d}$, then both $0\leq \abs{p_d}\pm\norm{\bv T-\bv c_J}$. If $p_d\geq 0$ then $\abs{p_d}=p_d$, and then neither option yields $K_z\leq 0$. If $p_d\leq 0$, then ${\abs{p_d}=-p_d}$ and this inequality becomes ${0\leq -p_d\pm\norm{\bv T-\bv c_J}}$, i.e. $0\geq p_d\mp\norm{\bv T-\bv c_J}$. Thus if $p_d\leq 0$, both choices lead to $K_z\leq 0$. Thus, under this additional assumption, $C$ points at $\bv T$ iff $P$ is given by one of
    \[
        \paran{\tfrac{\bv T- \bv c_J}{\norm{\bv T-\bv c_J}}\pm\uv z}\cdot(\bv x-\bv c_J)=0.
    \]
\end{enumerate}
The backward-pointing version of the Pointing Lemma \ref{lem:pointing} yields the analogous result. This results in replacing `$\leq$' with `$\geq$' wherever $Z_{\leq 0}$, $p_d\leq 0$, and $p_d\leq \tfrac{1}{2}T_z$ occur.
\end{proof}

Obtaining the midplane for Az/El pointing with constrained plunge distance is a simple application of our earlier efforts.

\begin{corollary}\label{cor:AE2MP-plunge}
Consider a configuration with midplane $P$, plunge distance $p_d\in \RR$, and distal normal $\uv N_D$. Then
\begin{enumerate}
    \item $\uv N_D\neq-\uv z$ iff $P$ can be given by $(\uv z+\uv N_D)\cdot(\bv x-\bv c_J)=0$.
    \item $\uv N_D=-\uv z$ iff $P$ contains $Z$.
\end{enumerate}
Holds for backward-normal by replacing $\uv N_D$ with $\shortminus\uv N_D$.
\end{corollary}

\begin{proof}
Follows immediately from Corollary \ref{cor:AE2MP-constrained} by letting ${\bv q=\bv c_J=p_d\uv z}$ and observing that the only planes parallel to $Z$ containing $p_d\uv z$ are the planes containing $Z$ itself.
\end{proof}

Though the case of $p_d=\infty$ is ommitted in the above results, we can quickly deduce that in that case the distal normal can only be $-\uv z$. To see this, note that $p_d=\infty$ only if the midplane $P$ is parallel to $Z$. Hence by midplane symmetry, $\uv N_D=R_P(-\uv z)=-\uv z$. So only targets $\bv T$ directly below the base plane can be pointed at when $p_d=\infty$, in which case the midplane must given by $(\bv T-T_z\uv z)\cdot(\bv x-\tfrac{\bv T+T_z\uv z}{2})=0$. 

One subtlety with plunge distance is that even with midplane symmetry, a finite plunge distance or well-defined center is not guaranteed.  As an example of an undefined center, consider the configuration in Figure \ref{fig:vertmidplane}.  Here, the midplane does not intersect the $z$-axis, so the plunge distance is $\infty$ and there is no center.  Physical constructions of the Canfield joint, including the design of the original prototype \cite{ganino}, may restrict the range of motion of the base joints, and thereby the configuration space, such that configurations with vertical midplanes are impossible. However, such limitations will remove the ability to point straight down.

\begin{figure}[ht]
\centering
\includegraphics[width=0.5\textwidth]{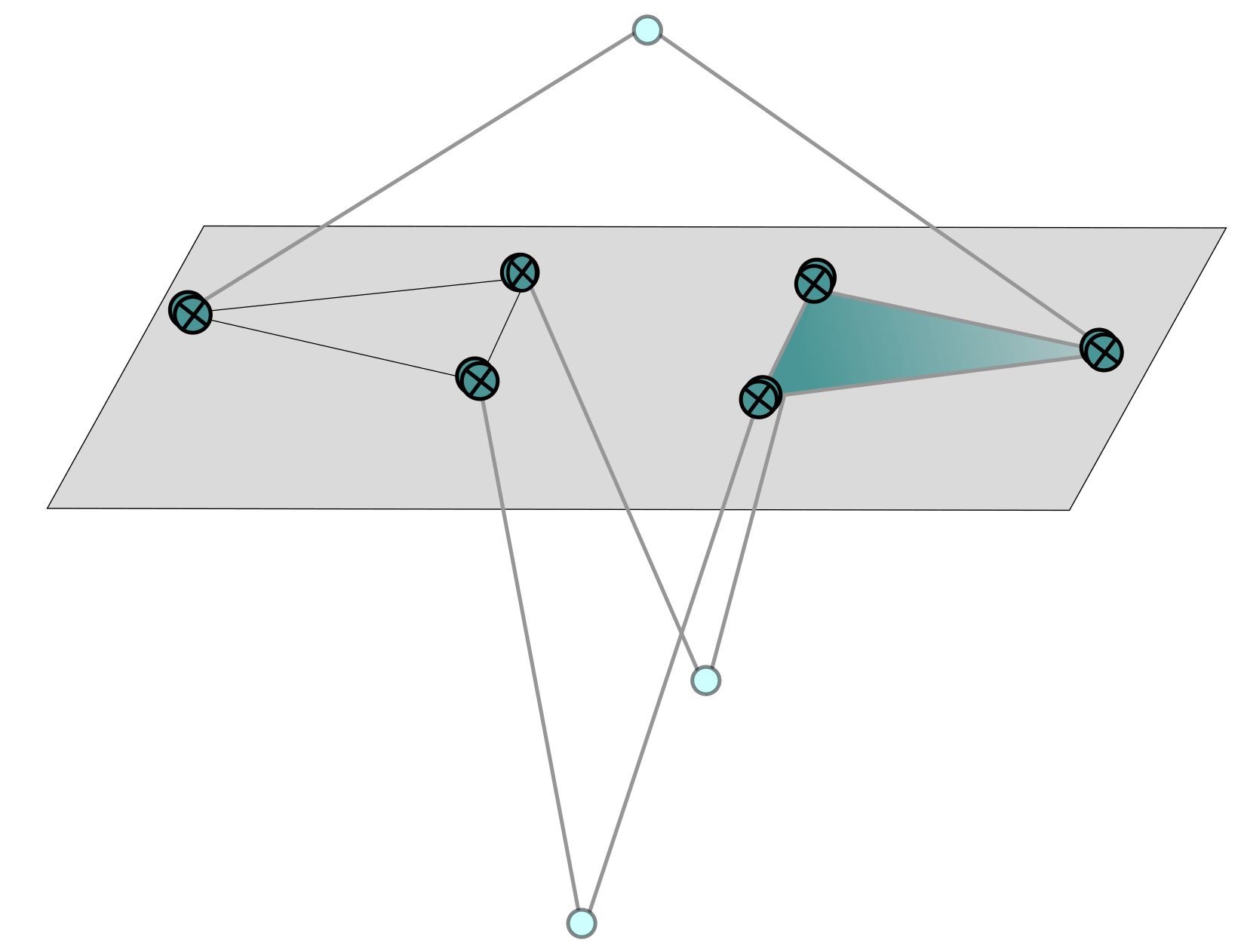}
\caption{A configuration with a vertical midplane and infinite plunge distance (the base triangle center is the base center). \label{fig:vertmidplane}}
\end{figure}

%%%%%%%%%%%%%%%%%%%%%%%%%%%%%%%%%%%%%%%%%%%%%%%%%%%%%%%%%%%%%%%%%%%%%%%%%%%%%%%%%%%%
\subsection{Pointing with a Frozen Midjoint and Fixed Plunge Distance}\label{sec:frozenplunge}

Finally, we revisit why one shouldn't simply use the traditional plunge distance based kinematics when stuck in a frozen midjoint failure mode. Let $P$ be a midplane of a Canfield joint configuration with frozen midjoint $\bv m^*$ and finite plunge distance $p_d\in\RR$. As we have seen, this means both $\bv m^*$ and $\bv c_J=p_d\uv z$ are on the midplane $P$. Additionally suppose $\bv m^*\neq \bv c_J$. Since $\bv m^*\in P$ and by midplane symmetry, $\bv m^*$ must be equidistant to $\bv D_c$ and $\bv B_c=\bv 0$. Thus, $\bv D_c\in S_{(\bv m^*,\bv 0)}$. Similarly, $\bv c_J$ is equidistant to $\bv D_c$ and $\bv B_c=\bv 0$, meaning $\bv D_c\in S_{(\bv c_J,\bv 0)}$. Thus $\bv D_c\in S_{(\bv m^*,\bv 0)}\cap S_{(\bv c_J,\bv 0)}$, i.e. the distal center is forced to lie on the intersection of two distinct spheres. 

The intersection of two distinct spheres is at best a circle. In practice, one can't even access all of this circle because the midplane-to-midjoint inverse kinematics further constrains what midplanes are viable. So we are working with some circular arc of possible distal centers. Thus by choosing to constrain the plunge distance while $\bv m^*$ is frozen, we have severely limited our pointing abilities and have only 1DoF for moving the distal plate.

If by sheer coincidence $\bv m^*$ lies on $Z$, then we can choose $p_d=\norm{\bv m^*}$. In this case $\bv c_J=\bv m^*$ and the two spheres from before are identical. As a result, the intersection is an entire sphere and this means that $\bv D_c$ has 2 degrees of (spherical) freedom. In this very special case, nothing is lost by enforcing the correct plunge constraint. This can be seen in the leftmost example of Figure \ref{fig:feaspointing_frozen}.

\section{Additional Consequences and Visualizations}\label{sec:viz}
In this section we explore some of the consequences of the inverse kinematics from previous section. In particular, we will introduce the distal normal field, the affine pointing locus, and the Az/El pointing locus.
The loci may be used to solve the inverse kinematic problem, as shown in Figure \ref{fig:IKflowchart}, and provide a tool for visualization.
% All of these have nice visualizations associated to them and the loci in particular can be somewhat thought of as a kind of inverse kinematics (and so we include them in Figure \ref{fig:IKflowchart}). 

\subsection{Distal Normal Field}

A priori, to determine the distal normal we would need to know the midplane. However, Theorem \ref{thm:DC2MP} will allow us to determine the distal normals directly from nonzero distal centers because such distal centers determine their own midplane. This gives rise to a vector field $\bv F(\bv x)$ which we call the \textbf{distal normal field}. It is given explicitly in the following theorem.

\begin{theorem}\label{thm:distalfield}
If a Canfield joint configuration has distal center $\bv x\neq\bv 0$, then the distal normal is given by
\[
    \bv F(\bv x) 
    = 2\frac{(\bv x\cdot\uv z)\bv }{\norm{\bv x}^2}\bv x - \uv z
    = 2(\hat{\bv x}\cdot\uv z)\hat{\bv x} - \uv z.
\]
Moreover, for $\guv\rho\perp\uv z$ 
\[
    \bv F(\sin\theta\guv\rho + \cos\theta\uv z) 
    = \sin 2\theta\guv\rho + \cos 2\theta\uv z.
\]
\end{theorem}
\begin{proof}
For a point $\bv x\neq \bv 0$ to be a distal center, the required midplane must be normal to $\bv x$ and pass through $\bv x/2$ (Theorem \ref{thm:DC2MP}). Thus, by Proposition \ref{prop:forwardkinematics}, the distal normal at $\bv x$ is given by $R_{[\bv x,\bv 0]}(-\uv z) = -\uv z + 2\tfrac{(\bv x\cdot\uv z)}{\norm{\bv x}^2}\bv x$ as desired. Finally, applying this formula with $\guv\rho\perp\uv z$,
\begin{align*}
    \bv F(\sin\theta\guv\rho + \cos\theta\uv z) 
    &=
    2(\cos\theta\uv z)(\sin\theta\guv\rho + \cos\theta\uv z) - \uv z \\
    &= 2\cos\theta\sin\theta\guv\rho + (2\cos^2\theta -1)\uv z \\
    &= \sin 2\theta\guv\rho + \cos 2\theta\uv z.
\end{align*}
\end{proof}

Since $\hat{\bv x}$ is the normal to the midplane associated to $\bv x$,\footnote{Where $\hat{\bv x}=\bv x/\norm{\bv x}$ is not to be confused with the unit vector $\uv x$.} the above shows that the distal normal doubles the polar angle. We visualize a 2D slice of this vector field in Figure \ref{fig:distalfield} and note a resemblance to the magnetic field of a dipole.\footnote{The magnetic field of a magnetic dipole with magnetic moment $\uv z$ is given by $3(\hatbv x\cdot \uv z)\hatbv x - \uv z$ times $\tfrac{\mu_0}{4\pi\norm{\bv x}^3}$. Comparing to $\bv F(\bv x)=2(\hatbv x\cdot\uv z)\hatbv x -\uv z$ we can see that these two are nearly the same when ignoring the scalar factor.}

\begin{figure}[ht]
\centering
\includegraphics[width=.45\textwidth]{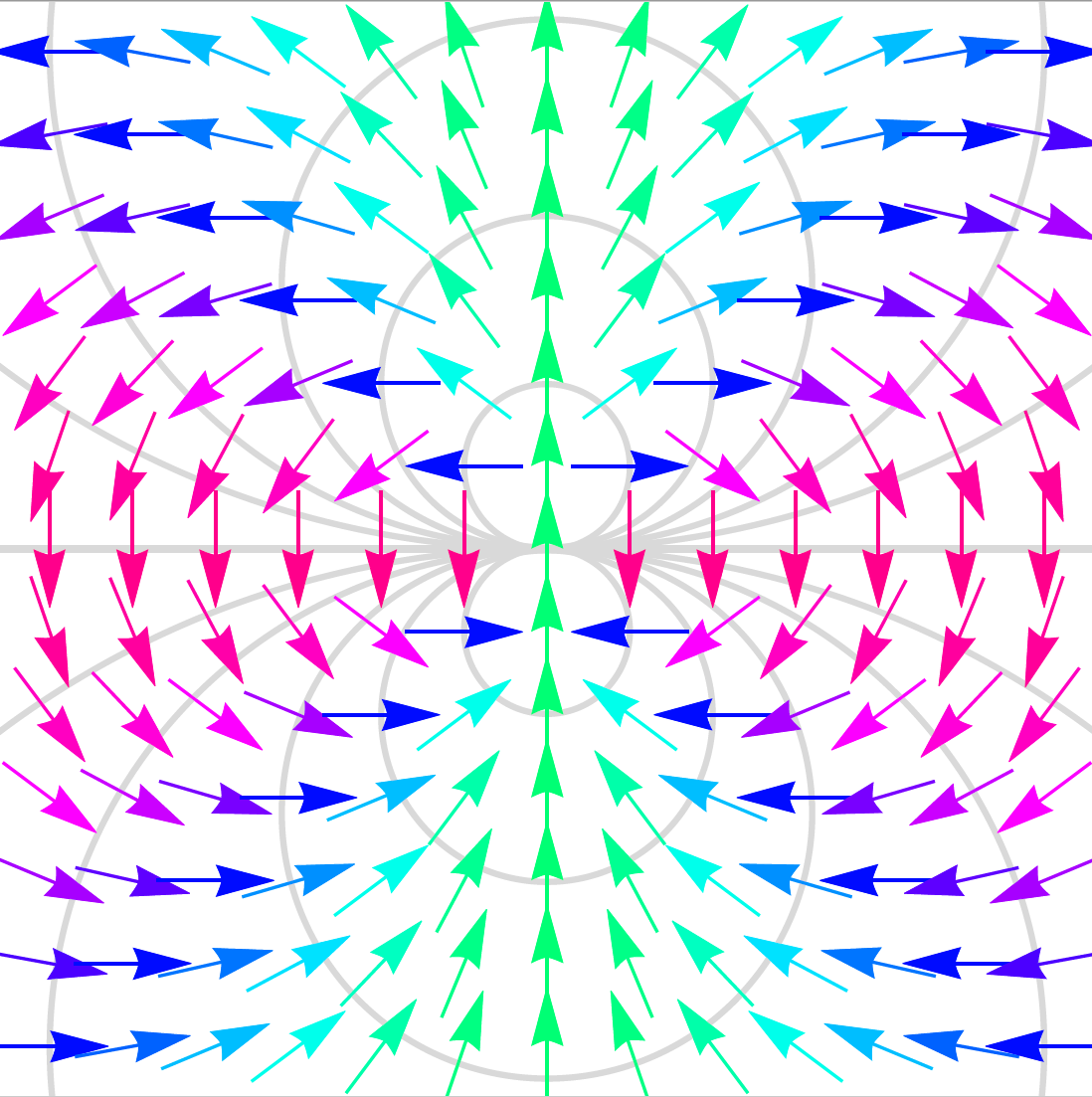}
\caption{2D slice of $\bv F(\bv x)$ on plane containing $Z$.}
\label{fig:distalfield}
\end{figure}

\subsection{Pointing Loci: Pointing Constraint to Distal Center}

\begin{figure*}[!htb]
\centering
\includegraphics[width=1\textwidth]{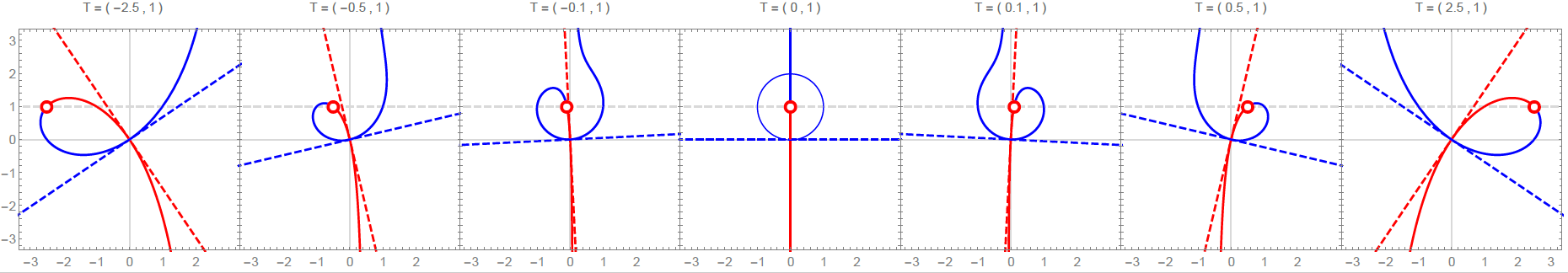}
\caption{The solid red, solid blue, dashed red, and dashed blue curves are $\Lambda^+_{\bv T}$, $\Lambda^-_{\bv T}$, $\tilde{\Lambda}^+_{\bv N}$, and $\tilde{\Lambda}^-_{\bv N}$ within $W$ respectively. We set $\bv T=\bv N$ and depict $\bv T$ with the hollow red dot. We vary $T_\rho$ from $-2.5$ to $2.5$ while keeping $T_z=1$ fixed. Note the transition from a single curve to a union of a circle and line. Also note that the Az/El locus is tangent to the affine locus at the origin.} 
\label{fig:affine+azelpointinglocus}
% \CBnote{Make text larger if necessary}
\end{figure*}

If we want a Canfield joint to point at an object $\bv T$ (affine pointing), or in a direction $\bv N$ given by azimuth and elevation (Az/El pointing), there is a continuum of distal center locations that would get the job done.  Here we will show how to find all the distal centers that accommodate such pointing objectives.

% \RSnote{WRITE:  Adjust exposition in Pointing Loci section - Key:  valid Distal Center location}
% When the authors first approached the problem of the frozen midjoint, it helped to understand the set of all possible ways to point at a fixed target $\bv T$. Though the arguments in earlier sections no longer explicitly use these ideas, the original approach relied heavily on understanding this ``pointing locus'' and so we include it here in the hope that it may help others.

\begin{definition}
Given a target point $\bv T$, let the \textbf{affine pointing locus} $\Lambda_{\bv T}$ be the set of distal centers $\bv D_c$ such that for some configuration $\bv T = \bv D_c + t\uv N_D$ for some $t \in \RR$. The \textbf{affine forward-pointing locus} $\Lambda^+_{\bv T}$ and \textbf{affine backward-pointing locus} $\Lambda^-_{\bv T}$ are the subset where $t\geq 0$ and $t\leq 0$ respectively. 
\end{definition}
\begin{remark}\label{rk:origininloci}
By this definition, $\bv B_c = \bv 0$ is always in $\Lambda_{\bv T}$ because the origin can point anywhere (see Remark \ref{rk:DCis0}). 
\end{remark}

Simply put, $\Lambda_{\bv T}$ is the set of places the distal-center can be so that $\bv T$ lies on the pointing axis of the configuration. This does not depend on a specific Canfield joint configuration and considers all possible Canfield joints at once. 

The following result provides necessary and sufficient conditions for identifying elements of $\Lambda_{\bv T}$.

\begin{theorem}\label{thm:AffinePointingLocus}
Let $\bv T=T_\rho\guv \rho + T_z\uv z$ where $\guv\rho\perp \uv z$, and let $W=\opn{span}\set{\guv\rho,\uv z}$ denote the $\rho z$-plane. Then $W\cap \Lambda_{\bv T}$ is the set of points $\rho\guv\rho + z\uv z$ which satisfy the equation
\[
    E_{\bv T}(\rho,z) \coloneqq \rho^3 + \rho z^2 + T_\rho(z^2-\rho^2)-2T_z \rho z = 0.
\]
Moreover, if $\bv T\notin Z$ then 
\[
    \Lambda_{\bv T} = \set{\rho\guv\rho+z\uv z\mid E_{\bv T}(\rho,z)=0}.
\]
\end{theorem}
\begin{proof}
By Remark \ref{rk:origininloci} we know that $\bv 0\in\Lambda_{\bv T}$, so assume $\bv x\neq\bv 0$. Note that $\bv x\in W\cap\Lambda_{\bv T}$ iff we can write $\bv x = \rho\guv\rho + z \uv z$ and $\bv T$ lies on the pointing axis of $\bv x$. By the Pointing Lemma \ref{lem:pointing} and distal-center inverse kinematics,  $\bv T$ will lie on the pointing axis if and only if the plane through $\bv x/2$ is normal to $\bv x$ (i.e. the midplane associated to $\bv x$) reflects $\bv T$ onto the $z$-axis. This is equivalent to saying that $R_{[\bv x,\bv x/2]}(\bv T)$ has no $\rho$-component, i.e. 
\[
    0=R_{[\bv x, \bv x/2]}(\bv T)\cdot\guv\rho 
    = \Big(\bv T - 2 \fr{(\bv T-\bv x/2)\cdot\bv x}{\norm{\bv x}^2}\bv x\Big)\cdot\guv\rho .
\]
Multiplying by $\norm{\bv x}^2$ and distributing the dot products, this is equivalent to
\[
    0 = T_\rho\norm{\bv x}^2 - 2 (\bv T\cdot \bv x)\rho +\norm{\bv x}^2 \rho  ,
\]
which is the same as
\[
     0 = T_\rho(\rho^2+z^2) - 2(T_\rho \rho + T_z z)\rho + (\rho^2 + z^2)\rho ,
\]
which when expanded reduces to $E_{\bv T}(\rho,z)=0$. 

Finally, if $\bv T\notin Z$ then $W=\opn{span}\set{\bv T,\uv z}$. It follows from the Pointing Lemma \ref{lem:pointing} that the pointing axis and $z$-axis must be reflections of each other over some plane. Hence these two lines must be coplanar. Since $\bv T\notin Z$ that means the plane containing the pointing axis and $Z$ is the plane containing $\bv 0, \uv z,$ and $\bv T$ which is $W$. Thus $\Lambda_{\bv T}\subseteq W$, so $\Lambda_{\bv T}\subseteq W\cap\Lambda_{\bv T}$. Since the reverse inclusion is automatic we have $\Lambda_{\bv T}=W\cap\Lambda_{\bv T}$.
\end{proof}

% \begin{figure}[!htb]
% \centering
% \subfloat[Graph of affine pointing locus for $\bv T = (4,6)$ \label{fig:affinepointinglocus}]{\includegraphics[width=0.38\textwidth]{Figures/Affine_Pointing_Locus.png}}\hfill
% \subfloat[Graph of AzEl pointing locus for $\bv N = (4,6)$ \label{fig:azelpointinglocus}] {\includegraphics[width=0.38\textwidth]{Figures/AzEl_Pointing_Locus.png}}\hfill
% \subfloat[The two above pointing loci for $\bv T = \bv N$ graphed together. \label{fig:affine+azelpointinglocus}] {\includegraphics[width=0.38\textwidth]{Figures/Affine+AzEl_Loci.png}}\hfill
% \caption{Illustrations of inverse kinematics forms.} \
% \end{figure}

% \begin{figure*}[!htb]
% \centering
% \subfloat[Graph of affine pointing locus for $\bv T = (4,6)$ \label{fig:affinepointinglocus}]{\includegraphics[width=0.32\textwidth]{Figures/Affine_Pointing_Locus.png}}\hfill
% \subfloat[Graph of AzEl pointing locus for $\bv N = (4,6)$ \label{fig:azelpointinglocus}] {\includegraphics[width=0.32\textwidth]{Figures/AzEl_Pointing_Locus.png}}\hfill
% \subfloat[Both pointing loci for $\bv T = \bv N$ graphed together. \label{fig:affine+azelpointinglocus}] {\includegraphics[width=0.32\textwidth]{Figures/Affine+AzEl_Loci.png}}\hfill
% \caption{Pointing loci.} \
% \end{figure*}

% \begin{figure}[ht]
% \centering
% \includegraphics[width=0.5\textwidth]{Figures/Affine_Pointing_Locus.png}

% \caption{Graph of affine pointing locus for $\bv{T} = (6, 4)$\label{fig:affinepointinglocus}
% }
% \end{figure}

Note that $\Lambda_{\bv T}$ does not depend on the design parameters $\ell_i$, $b_i$, and $\uv N_i$, and only depended on midplane symmetry and that $\bv B_c=\bv 0$. The next result lets us understand the geometry of the affine pointing locus and provides a parametrization.

\begin{theorem}\label{thm:AffineLocusParametrized}
If $\bv T\in Z$ then $\Lambda_{\bv T}=Z\cup S_{(\bv T, \bv 0)}$. When $\bv T\notin Z$, then $\Lambda_{\bv T}$ is a planar nodal cubic curve passing through $\bv T$, with a node at $\bv 0$, and can be rationally parametrized as
\[
    \bv r(t) = \paran{\fr{\norm{\bv T}^2 - t^2}{\norm{\bv T - t\uv z}^2}}(\bv T-t\uv z),\qquad t\in\RR.
\]
Moreoever, $\Lambda^+_{\bv T}$ corresponds to $t\leq 0$, and $\Lambda^-_{\bv T}$ to $t\geq 0$.
\end{theorem}
\begin{proof}
If $\bv T\in Z$ then $T_\rho=0$ and $\bv T=T_z\uv z$. Thus,
\[
    0=E_{\bv T}(\rho,z)= \rho^3 + \rho z^2-2T_z \rho z = \rho\!\paran{\rho^2\!+\!(z\!-\!T_z)^2\!-\!T_z^2}\!.
\]
This is zero iff $\rho=0$ or $\rho^2+(z-T_z)^2=T_z^2$. By Theorem \ref{thm:AffinePointingLocus} and cylindrical symmetry, this means $\Lambda_{\bv T}=Z\cup S_{(\bv T, \bv 0)}$. 

Now suppose $\bv T\notin Z$ and consider $t\uv z\in Z$ for $t\in\RR$. By Theorem \ref{thm:AP2MP}, the planes that point/backward-point at $\bv T$ are precisely the ones normal to $\bv N=\bv T-t\uv z$ containing the midpoint $\bv q=\tfrac{1}{2}(\bv T+t\uv z)$ for $t\leq 0$ and $t\geq 0$ respectively. The distal centers associated to these planes are given by
\[
    \bv r(t) = R_P(\bv 0) = 2\frac{\bv N\cdot\bv q}{\norm{\bv N}^2}\bv N 
    = \paran{\fr{\norm{\bv T}^2 - t^2}{\norm{\bv T - t\uv z}^2}}(\bv T-t\uv z).
\]
Thus, this curve is precisely $\Lambda_{\bv T}$ where $\Lambda^+_{\bv T}$ and $\Lambda^-_{\bv T}$ correspond to  $t\leq 0$ and $t\geq 0$ respectively.

Finally, since $\bv T\notin Z$, we know $\Lambda_{\bv T}$ is a planar and given by a polynomial of third degree by Theorem \ref{thm:AffinePointingLocus}. This cubic curve has a node at $\bv 0$ because $\bv r(t)=\bv 0$ has two solutions of $t = \pm\norm{\bv T}$, but it is one-to-one everywhere else.
\end{proof}

For Az/El pointing problems, we need a different pointing locus. We mimic the affine pointing locus results below.

\begin{definition}
Let $\bv N \neq 0$ and let $\uv{N} = \tfrac{\bv N}{\norm{\bv N}}$.  The \textbf{Az/El forward-pointing locus} $\tilde{\Lambda}^+_{\bv N}$ is the set of distal centers where $\uv N_D = \uv{N}$ for some configuration. Similarly, the \textbf{Az/El backward-pointing locus} $\tilde{\Lambda}^-_{\bv N}$ is the set of distal centers where $\uv N_D = \shortminus\uv{N}$ for some configuration (i.e. the configuration has distal backward-normal $\uv N$). Their union is the \textbf{Az/El pointing locus} $\tilde{\Lambda}_{\bv N}=\tilde{\Lambda}^+_{\bv N}\cup \tilde{\Lambda}^-_{\bv N}$. 
\end{definition}

\begin{remark}\label{rk:AzEllocustrick}
Note that $\tilde\Lambda^+_{-\bv N} = \tilde\Lambda^-_{\bv N}$ and $\tilde\Lambda^+_{\alpha\bv N}=\tilde\Lambda^+_{\bv N}$ for all $\alpha>0$. So it will suffice to know $\tilde\Lambda^+_{\uv N}$ for all unit vectors $\uv N$. 
\end{remark}
\begin{remark}\label{rk:AzEllocustrick2}
A configuration with distal center $\bv x$ and distal normal $\tfrac{\bv N}{\|\bv N\|}$ points at $\bv x+\bv N$ and backward-points at $\bv x-\bv N$. As a result we can relate the Az/El and affine pointing loci by: $\bv x\in\tilde{\Lambda}_{\bv N}$ iff $\bv x\in\Lambda_{\bv x+\bv N}$ when $\bv N\neq\bv 0$.
\end{remark}
The following uses Theorem \ref{thm:AffinePointingLocus} and Remark \ref{rk:AzEllocustrick2} to give necessary and sufficient criteria for belonging to $\tilde\Lambda_{\bv N}$.

\begin{theorem}\label{thm:AzElPointingLocus}
Let $\bv N=N_\rho\guv\rho + N_z\uv z\neq\bv 0$ where $\guv\rho\perp\uv z$, and let $W = \operatorname{span} \{\guv\rho , \uv z \}$ denote the $\rho z$-plane.  Then $W \cap \tilde{\Lambda}_{\bv N}$ is the set of points $\rho \guv\rho + z \uv z$ given by 
\[
    \tilde{E}_{\bv N}(\rho,z) := N_\rho(z^2-\rho^2) - 2N_z \rho z =0.
\]
Moreover, if $\bv N\notin Z$, then
\[ 
    \tilde{\Lambda}_{\bv N} = \{ \rho \guv \rho + z \uv z \, | \, \tilde{E}_{\bv N}(\rho,z)=0 \}.
\]
\end{theorem}
\begin{proof}
By Remark \ref{rk:AzEllocustrick2}, we can deduce that $\bv x\in W \cap \tilde{\Lambda}_{\bv N}$ iff $\bv x\in W \cap \Lambda_{\bv x+\bv N}$. We know $\bv x\in W$ means $\bv x = \rho \guv\rho + z \uv z$ and so $\bv x+\bv N=(\rho+N_\rho)\guv\rho+(z+N_z)\uv z$. This together with Theorem \ref{thm:AffinePointingLocus} tells us that $\bv x\in W \cap \Lambda_{\bv x+\bv N}$ iff $\bv x = \rho \guv\rho + z \uv z$ and 
\[
    0\!=\!E_{\bv x+\bv N}(\rho,z)\!=\!\rho^3 + \rho z^2 + (\rho+N_\rho)(z^2-\rho^2)-2(z+N_z) \rho z.
\] 
Simplifying this, we get the desired equation
\[
    0 = N_\rho(z^2-\rho^2) - 2N_z \rho z.
\]  
Finally, if $\bv N \notin Z$, then we can use a similar argument as in the proof of Theorem \ref{thm:AffinePointingLocus} to show that $\tilde{\Lambda}_{\bv N} = W \cap \tilde{\Lambda}_{\bv N}$.
\end{proof}
% \begin{proof}
% Suppose $\bv x \in W \cap \tilde{\Lambda}_{\bv N}$, so we can write $\bv x = \rho \guv\rho + z \uv z$.  Since $\bv x \in \tilde{\Lambda}_{\bv N}$, $\bv x$ a distal center such that $\uv N_D = \uv N$.  Let $P$ be the midplane guaranteed by Theorem \ref{thm:DC2MP}.  Note that a Canfield joint with $\bv x$ as its distal center and where $\uv N_D = \uv N$ must point at $\bv x + \bv N$, so by the Pointing Lemma $R_P(\bv x + \bv N) \in Z$.  This is equivalent to
% \[ 
% 0 = R_P(\bv x + \bv N) \cdot \guv\rho  = \left( \bv x + \bv N - 2 \frac{(\bv x/2 + \bv N) \cdot \bv x}{\norm{\bv x}^2}\bv x  \right) \cdot \guv\rho.
% \]
% Multiplying by $\norm{\bv x}^2$ and distributing dot products yields
% \[ 
% 0 = \norm{\bv x}^2\bv N\cdot \guv\rho + 2(\bv N \cdot \bv x)(\bv x\cdot \guv\rho).
% \]
% Since everything is in $W$, we can rewrite this as
% \[ 
% 0 = (\rho^2 + z^2)N_\rho - 2(N_\rho \rho + N_z z)\rho = N_\rho(z^2-\rho^2) - 2N_z \rho z
% \]

% \noindent which is the equation $\tilde{E}_{\bv N}(\rho,z)=0$.

% Finally, if $\bv N \notin Z$, then we can use a similar argument as in the proof of Theorem \ref{thm:AffinePointingLocus} to show that $\tilde{\Lambda}_{\bv N} = W \cap \tilde{\Lambda}_{\bv N}$.
% \end{proof}

% \begin{figure}[ht]
% \centering
% \includegraphics[width=0.5\textwidth]{Figures/AzEl_Pointing_Locus.png}

% \caption{Graph of AzEl pointing locus for $\bv{N} = (6, 4)$\label{fig:azelpointinglocus}
% }
% \end{figure}

Examining the graphs of the Az/El pointing locus in Figure \ref{fig:affine+azelpointinglocus}, we see it is composed of two lines.  These are given by 
\[
N_\rho \rho + (N_z +\norm{\bv N})z=0, \quad N_\rho \rho + (N_z - \norm{\bv N})z=0.
\]
when $N_\rho\neq 0$. We can check this by computing their product:
\begin{align*}
    (N_\rho \rho + (N_z + \norm{\bv N})z)\, (N_\rho \rho + (N_z - \norm{\bv N})z) \\
    = N_\rho^2 (\rho^2-z^2) +2N_\rho N_z \rho z.
\end{align*}
Setting this to 0 and multiplying by $-1$ yields:
\[
0 = N_\rho^2 (z^2-\rho^2)-2N_\rho N_z \rho z = N_\rho \tilde{E}_N(\rho,z).
\]

When $N_\rho \neq 0$, this is equivalent to $\tilde{E}_N(\rho,z)=0$. Conveniently, one of these lines is the forward-pointing locus $\tilde{\Lambda}_{\bv N}^+$ and the other is the backward-pointing locus $\tilde{\Lambda}_{\bv N}^-$.  Below, we shall give an alternate parametrization of these two lines.

\begin{theorem}\label{thm:AzElLocusParametrization}
Both $\tilde\Lambda^+_{-\uv z}$ and $\tilde\Lambda^-_{\uv z}$ are the base plane. For the other cases, $\tilde\Lambda^+_{\bv N}$ and $\tilde\Lambda^-_{\bv N}$ are orthogonal lines through the origin which can parametrized respectively by
\begin{align*}
    \tilde{\bv r}^+(t) 
    &=  t\paran{\uv z + \tfrac{\bv N}{\norm{\bv N}}}, \quad t\in\RR,\, \tfrac{\bv N}{\norm{\bv N}}\neq -\uv z,\\
    \tilde{\bv r}^-(t) 
    &=  t\paran{\uv z - \tfrac{\bv N}{\norm{\bv N}}}, \quad t\in\RR,\, \tfrac{\bv N}{\norm{\bv N}}\neq \uv z.
\end{align*}
\end{theorem}

\begin{proof}
By Theorem \ref{thm:AE2MP}, a configuration has distal normal $-\uv z$ iff the midplane is parallel to $Z$. Forward kinematics implies such midplanes yield distal centers in the base plane. Thus $\tilde\Lambda^+_{-\uv z}=\tilde\Lambda^-_{\uv z}$ is the base plane.

Next let $\uv N:=\tfrac{\bv N}{\norm{\bv N}}\neq -\uv z$. By Theorem \ref{thm:AE2MP}, a configuration has distal normal $\uv N\neq -\uv z$ iff the midplane is orthogonal to $\uv z+\uv N$. Thus, midplanes for such configurations have midplanes of the form $(\uv z+\uv N)\cdot(\bv x-\bv q)=0$ for some $\bv q$. Forward kinematics implies such midplanes yield distal centers: 
\[
    \bv D_C = R_{[\uv z+\uv N, \bv q]}(\bv 0) = 2\frac{(\uv z+\uv N)\cdot\bv q}{\|\uv z +\uv N\|^2}(\uv z + \uv N).
\]
Since $\bv q$ was arbitrary, the coefficient of $(\uv z +\uv N)$ can be any real number. Thus it follows $\tilde{\bv r}^+(t)=t(\uv z +\uv N)$ parametrizes $\tilde\Lambda^+_{\bv N}$. By Remark \ref{rk:AzEllocustrick}, we know  $\tilde\Lambda^-_{\bv N}=\tilde\Lambda^+_{-\bv N}$, so it follows that $\tilde{\bv r}^-(t)=t(\uv z-\uv N)$ parametrizes $\tilde\Lambda^-_{\bv N}$ for $\uv N\neq \uv z$.
\end{proof}

The observant reader may notice that the two pointing loci equations are quite similar and in fact, they only differ by $\rho^3+\rho z^2$ if we replace $T_\rho$ and $T_z$ with $N_\rho$ and $N_z$ respectively. Additionally, in Figure \ref{fig:affine+azelpointinglocus} we can see that for a given point $\bv T$, the lines tangent to $\Lambda_{\bv T}$ at the origin are precisely the lines given by $\tilde\Lambda_{\bv T}$. This is not a coincidence. One can show directly via projective geometry that $\Lambda_{\alpha \bv T}\to \tilde\Lambda_{\bv T}$ as $\alpha\to\infty$.\footnote{In projective coordinates (also known as homogenous coordinates), $\alpha\bv T$ is given by $[\alpha T_\rho:\alpha T_z:1]= [T_\rho:T_z:\tfrac{1}{\alpha}]$. The point at infinity obtained by taking $\alpha\to\infty$ is $\bv T^*:=[T_\rho:T_z:0]$.  Extending $E_{\bv T}=0$ to projective space (w.r.t $\bv T$) yields $T_w(\rho^3 + \rho z^2) + T_\rho(z^2-\rho^2)-2T_z \rho z = 0$ for points $[T_\rho:T_z:T_w]$. Plugging in $\bv T^*$ yields $T_\rho(z^2-\rho^2)-2T_z \rho z = 0$ which is the Az/El locus $\tilde E_{\bv T}(\rho,z)=0$.} %\RSnote{Desmos link for this?}
This can be interpreted as saying that the Az/El locus becomes the affine locus when $\bv T$ is infinitely far away.  Alternatively, one can show this by computing tangent slopes using the results from Theorem \ref{thm:AzElPointingLocus} and Theorem \ref{thm:AffineLocusParametrized}.

\subsection{Feasible Distal Normals for the Standard Canfield Joint}\label{sec:feasibleDN}
Let's consider a standard Canfield joint and recall that in this case we use the center of the equilateral base plate as the base center $\bv B_c$. Moreover, since we are dealing with a standard Canfield joint, the midcircle planes $P_i$ all intersect in a common line, namely $Z$.

In Section \ref{sec:MP2MJ} we describe how to obtain the midjoints from a midplane. In particular, Lemma \ref{lem:Midplane2MidjointDiscriminant} tells us that 
\[
    \Delta_i=\big(\tfrac{\bv \nu_i}{\norm{\bv\nu_i}}\cdot\bv\delta_i \big)^2 - \norm{\bv\delta_i}^2 + \ell_i^2 \geq 0
\]
is the necessary and sufficient criterion for determining whether a plane $P$ given by $\bv N\cdot(\bv x-\bv q)=0$ will intersect $M_i$ (assuming $P$ is not parallel to $P_i$). 

Multiplying both sides by $\norm{\bv\nu_i}^2$ and applying Lagrange's identity yields the constraint,
% \[
%     \paran{\bv\nu_i\cdot\bv\delta_i}^2 - \norm{\bv\nu_i}^2\norm{\bv\delta_i}^2 + \ell_i^2 \norm{\bv\nu_i}^2\geq 0.
% \]
% Using Lagrange's identity we obtain
\[
    \norm{\bv\nu_i}^2\Delta_i=\ell_i^2 \norm{\bv\nu_i}^2 - \norm{\bv\nu_i\times\bv\delta_i}^2 \geq 0.
\]  
This condition is satisfied iff $\bv\nu_i=\bv 0$ or $\Delta_i\geq 0$.\footnote{Note $\bv\nu_i=\bv 0$ iff $P$ is parallel to $P_i$. Since $P_i$ contains $Z$, this means $P$ is parallel to $Z$ and so $\bv\nu_i=\bv 0$ implies $\uv N_D=-\uv z$ by Theorem \ref{thm:AE2MP}.}

Recall that $\bv\nu_i=\bv N\times\uv N_i$ and $\bv\delta_i=\bv q_i-\bv B_i$ where $\bv q_i$ is any point on $P\cap P_i$. We can replace $\ell_i$ with $\ell$ since all arm lengths are equal in a standard Canfield joint. Thus we obtain
\[
    \ell^2 \norm{\bv N\times\uv N_i}^2 - \norm{(\bv N\times\uv N_i)\times(\bv q_i-\bv B_i)}^2 \geq 0,
\]  
and have such a constraint for each $i=1,2,3$. 

Now we will apply this to both the frozen midjoint failure mode and to the plunge constrained model with $p_d\in\RR$. This means we assume $P$ is given by $\bv N\cdot(\bv x-\bv q)=0$ where $\bv q=\bv m^*$ or $\bv q=p_d\uv z$. In the failure mode case, $\bv q_i$ will be the projection of $\bv m^*$ onto the line $P\cap P_i$ (see Remark \ref{rk:qi}). In the plunge case, $\bv q_i$ can be chosen to be $p_d\uv z$ for all $i$ (because all $P_i$ intersect at $Z$).

By Theorem \ref{thm:distalfield} we know that the polar angle for $\uv N_D$ is twice that of the polar angle of $\bv D_c$, and by Theorem \ref{thm:DC2MP} we also know that $\bv D_c$ is a normal to the midplane (when $\bv D_c\neq\bv 0)$. Thus, in this assumed context, these results allow us to take a candidate distal normal $\uv N_D=\sin\theta\guv\rho+\cos\theta\uv z$, and produce the associated midplane $\uv N\cdot(\bv x-\bv q)=0$ where $\uv N=\sin\tfrac{\theta}{2}\guv\rho+\cos\tfrac{\theta}{2}\uv z$. Converting $\uv N_D$ into $\uv N$ and then testing with the above criteria for each $i=1,2,3$, gives us a way of determining exactly which distal normals are achievable when accounting for the geometry of the base half-joint. If  $\norm{\bv\nu_i}^2\Delta_i\geq 0$ for all $i=1,2,3$, then this distal normal is achievable. If even one fails, then it is not. Again note this takes into account the sizes of the midcircles.

Using this pipeline, we visualize pointing capabilities of plunge constrained standard Canfield joints for different choices of parameters in Figure \ref{fig:feaspointing}. Similarly, we visualize the pointing capabilities of a standard Canfield joint in different frozen midjoint failure modes in Figure \ref{fig:feaspointing_frozen}. 

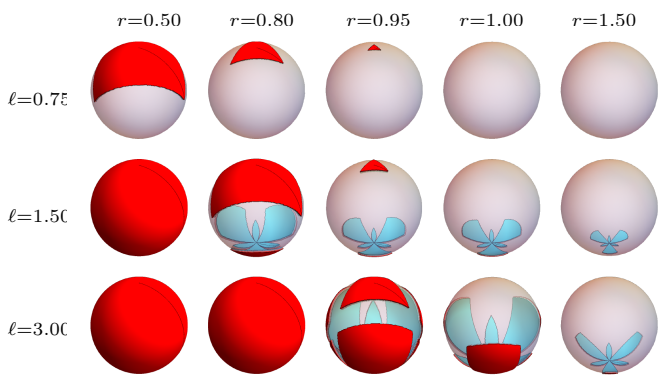
\begin{figure}[h]
\begin{tikzpicture}
\node [anchor=west] (l1) at (-1,1.55) {${\scriptstyle \ell = 0.75}$};
\node [anchor=west] (l2) at (-1,0) {${\scriptstyle \ell = 1.50}$};
\node [anchor=west] (l3) at (-1,-1.5) {${\scriptstyle \ell = 3.00}$};

\node [anchor=west] (t1) at (0.45,2.6) {${\scriptstyle r = 0.50}$};
\node [anchor=west] (t1) at (1.95,2.6) {${\scriptstyle r = 0.80}$};
\node [anchor=west] (t1) at (3.5,2.6) {${\scriptstyle r = 0.95}$};
\node [anchor=west] (t1) at (5,2.6) {${\scriptstyle r = 1.00}$};
\node [anchor=west] (t1) at (6.5,2.6) {${\scriptstyle r = 1.50}$};

\begin{scope}[xshift=4cm]
    \node[anchor=center,inner sep=0] (image) at (0,0) {\includegraphics[width=0.45\textwidth]{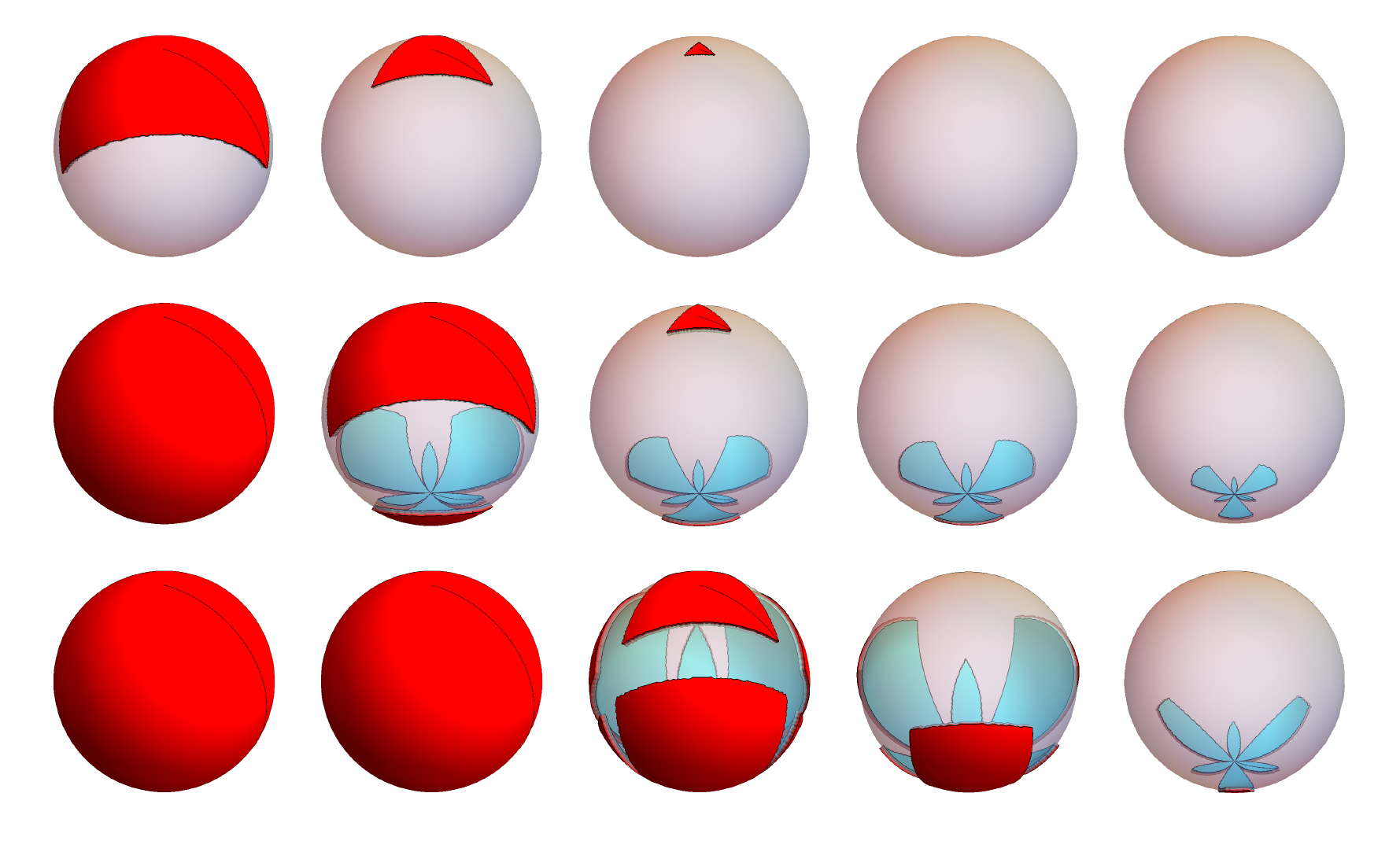}};
\end{scope}
\end{tikzpicture}
\caption{Here $\ell$ is arm length, $p_d=r\ell$, and $b=\sqrt{3}$ (so the distance from $\bv B_i$ to the base triangle center is $1$). The covered regions of the sphere indicate which distal normal orientations are possible with those parameters (the color code is as in Figure \ref{fig:feaspointing_frozen}). Note that when $r\geq 1$ (i.e. $p_d\geq \ell$), it is impossible to point up, but for sufficiently long arms it is still possible to point down. Also note that the ability to point down only appears after $\ell\geq 1$. As usual, this is under the assumption that the elements can pass through each other. }
\label{fig:feaspointing}
\end{figure}

\section{Future work}\label{sec:future}

The foregoing presents the most general and rigourous kinematic analysis of the Canfield joint to date. This was in large part feasible due to the key assumption of midplane symmetry. However, it is known that this condition can be violated in special cases (see \cite{KCThesis}). Thus, further research is warranted into what circumstances lead to loss of midplane symmetry and into what effect that may have on the kinematics. Such questions will require a thorough study of the full configuration space of the Canfield joint. To date, no such study has been conducted.

An additional reason to study the full configuration space of the Canfield joint is to shed light on its singularities. On the one hand, in line with the work in \cite{Canfield1998}, we conjecture that the set of singular base settings is very small and specifically is of measure zero. On the other hand, empirical evidence suggest that even being near a singular configuration can prove deleterious in physical applications. Ideally, future investigations would lead to criteria by which dangerous base settings can be identified and avoided - either by software control or by adjusting the physical design of the Canfield joint.  Conversely, dimensions and configurations of the Canfield joint may be designed to confer extraordinary pointing capability, as in the infinite plunge distance example shown in Figure \ref{fig:vertmidplane}, which holds promise for novel mechanisms.

% A thorough investigation of the configuration space, with particular focus on determining criteria for kinematic singularities, will inform robust motion planning for a wide variety of physical constructions. 

% Singular configurations, in which the robot cannot be reliably controlled, tend to occur at sets of base angles where midplane symmetry can be violated, as described in \cite{KCThesis}. \CBnote{More investigation of midplane symmetry violations} Avoiding these configurations, whether through software control or by adjusting the physical design of the Canfield joint to restrict the range of motion, is desirable in gimbal applications.  

% More work is warranted, however, in the extending and application of this analysis to physical robotic control. 

Motion planning and control of a prototype Canfield joint was validated for the mission needs of a deep space optical communication use case, and the methods and results of this analysis will be discussed in a subsequent paper.  However, the frozen midjoint case described above has not yet been integrated into motion planning algorithms for the Canfield joint or tested in a laboratory setting.

Finally, attention is also deserved by several concerns affecting physical robots which are not accounted for in this model:  for example, the limitations of the physical midjoint linkages will limit the configuration space accessible to the robot and have important implications for future analysis of the dynamics.  

\bibliographystyle{IEEEtran}
\bibliography{references.bib}

\end{document}